\documentclass{article}
\PassOptionsToPackage{numbers, compress}{natbib}
\usepackage[main,final]{neurips_2026}

\usepackage[utf8]{inputenc} 
\usepackage[T1]{fontenc}    
\usepackage[colorlinks=true,linkcolor=blue,citecolor=blue,urlcolor=blue,filecolor=blue,menucolor=red]{hyperref}       
\usepackage{url}            
\usepackage{booktabs}       
\usepackage{array}          
\usepackage{amsfonts}       
\usepackage{nicefrac}       
\usepackage{microtype}      
\usepackage{xcolor}         
\usepackage{caption}
\usepackage{amsmath, amssymb, amsthm}
\usepackage{bm}
\usepackage{mathtools}
\usepackage{algorithm}
\usepackage{algpseudocode}

\usepackage{enumitem}
\setlist[itemize]{leftmargin=*,topsep=2pt,itemsep=0pt,parsep=0pt,partopsep=0pt}
\setlist[enumerate]{leftmargin=*,topsep=2pt,itemsep=0pt,parsep=0pt,partopsep=0pt}

\usepackage{graphicx}
\usepackage{adjustbox}
\usepackage{capt-of}
\graphicspath{{figures/}{figures/samples/}}

\newtheorem{theorem}{Theorem}
\newtheorem{lemma}{Lemma}
\newtheorem{proposition}{Proposition}
\newtheorem{corollary}{Corollary}

\newtheorem{assumption}{Assumption}
\theoremstyle{remark}
\newtheorem{remark}{Remark}
\theoremstyle{definition}

\def\A{\mathbf{A}}                
            \def\I{\mathbf{I}}

\def\cF{\mathcal{F}}                
            \def\cN{\mathcal{N}}    \def\cO{\mathcal{O}}

\newcommand{\RR}{\mathbb{R}}

\newcommand{\EE}{\mathbb{E}}
\newcommand{\PP}{\mathbb{P}}

\def\eps{\varepsilon}

\DeclareMathOperator{\KL}{KL}

\DeclareMathOperator{\Var}{Var}
\DeclareMathOperator{\Cov}{Cov}

\DeclareMathOperator*{\argmin}{arg\,min}

\newcommand{\dd}{\mathrm{d}}

\DeclareMathOperator{\osc}{osc}

\newcommand{\ours}{\textsc{BridgeTwist}}

\newcommand{\vthe}{v_\theta}                   
\newcommand{\sthe}{s_\theta}                   
\newcommand{\vstar}{v^\star}                   
\newcommand{\sstar}{s^\star}                   

\def\I{I}
\def\A{A}
\newcommand{\hmu}[1]{\hat{\mu}_{#1}}       
\newcommand{\mustar}[1]{\mu^\star_{#1}}        
\newcommand{\hh}{\hat{h}}                      
\newcommand{\Ly}{L_y}                           
\newcommand{\Pos}{\pi^\star}                   
\newcommand{\PosN}{\pi^{\star,(N)}}            
\newcommand{\TT}{\tau}                         

\title{\textsc{BridgeTwist}: Twisting Schr\"odinger Bridges for Training-Free Conditional Sampling}

\author{%
  Mengyu Li, Qianqian Qu, Jun S. Liu\thanks{To whom correspondence should be addressed.} \\
  Department of Statistics and Data Science, Tsinghua University \\
  \texttt{mengyuli@tsinghua.edu.cn \quad quqq22@mails.tsinghua.edu.cn} \\
  \texttt{junsliu@tsinghua.edu.cn}
}

\begin{document}
\addtocontents{toc}{\protect\setcounter{tocdepth}{-1}}

\maketitle

\begin{abstract}
Generative models are often required to serve many downstream conditional tasks.
In the training-free setting, a pretrained unconditional model is reused as a prior and combined with a task-specific likelihood to sample from the corresponding posterior.
While such training-free samplers have been developed for diffusion and flow models, their counterpart for Schr\"odinger bridge (SB) models remains unexplored, despite the growing appeal of SBs as flexible stochastic transport models between general source and target distributions. We propose \textsc{BridgeTwist}, a training-free conditional sampler for pretrained SB generative models, formulated as a tempered twisted sequential Monte Carlo scheme.
The twist function employs a plug-in surrogate built on a closed-form endpoint predictor, which is derived by combining the velocity and score fields obtained from the SB pretraining.
Since this surrogate is less reliable near the source, we temper the twist function with a time-varying exponent, downweighting its influence at early steps.
We further introduce a bridge-shaped inference-time diffusion schedule that concentrates the proposal noise in the interior of the path.
Telescoping importance weights then cancel all intermediate twist factors, yielding an asymptotically exact sampler for the posterior.
Empirically, \textsc{BridgeTwist} outperforms training-free baselines on class-conditional and text-prompted image generation, and is competitive with specialized inverse problem solvers on inpainting, deblurring, and super-resolution.
\end{abstract}

\section{Introduction}\label{sec:intro}

\begin{figure}[t]
  \centering
  \includegraphics[width=\textwidth]{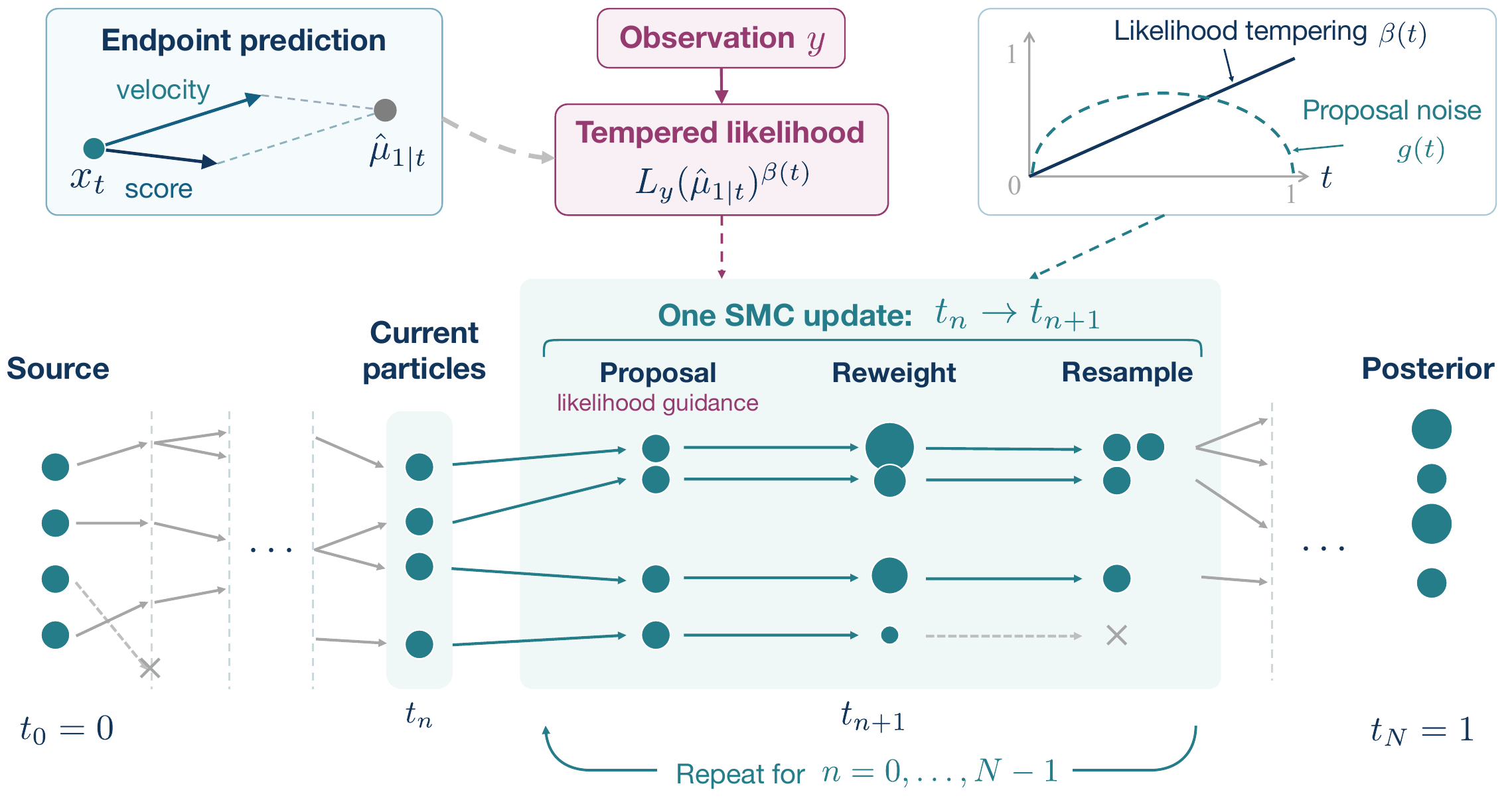}
  \caption{\ours{} overview. Schr\"odinger bridge's velocity and score predict endpoints for likelihood evaluation. The schedules $\beta(t)$ and $g(t)$ control likelihood tempering and proposal noise, respectively. Vertical position denotes state and circle area denotes weight. Resampling copies retain their states and are offset horizontally for visibility. The highlighted SMC update repeats along the time grid, and the terminal weighted particles approximate the target conditional distribution.}
  \label{fig:bridgetwist-overview}
\end{figure}

Conditional sampling draws from a generative model subject to specific constraints.
In the training-free regime, a pretrained unconditional model is treated as a prior $p_\theta(x)\approx p_\text{data}(x)$ and combined with a tractable likelihood $\Ly(x)=p(y|x)$.
Bayes' rule then defines the target posterior $p_\theta(x|y)\propto p_\theta(x)\Ly(x)$, allowing a single pretrained model to support many downstream conditional tasks without retraining.
Different choices of $\Ly$ recover classifier-guided sampling, inverse problems, text-guided generation, and protein motif scaffolding~\citep{wu2023practical,chung2023diffusion, dou2024diffusion, young2025tdsgenie, qu2026tftf}.

For score-based diffusion and flow matching priors, training-free conditional sampling has been developed along two main lines.
Gradient guidance and projection methods modify the unconditional sampling dynamics with likelihood-driven corrections at each step~\citep{chung2023diffusion, song2024fmps, kawar2022ddrm, song2023pseudoinverse, wang2025ocflow, martin2025pnpflow}; they are simple and effective, but generally introduce irreducible approximation error.
\emph{Sequential Monte Carlo} (SMC) samplers instead propagate a set of particles with reweighting and resampling, yielding consistent posterior samples and often stronger performance in various conditional tasks~\citep{trippe2023diffusion, wu2023practical, cardoso2024particle, dou2024diffusion, skreta2025fkcorrectors, singhal2025feynman, qu2026tftf}.
Across these training-free methods, an important component is a model-specific endpoint predictor from the current intermediate state; e.g., it is used to form lookahead twists in SMC samplers.

\emph{Schr\"odinger bridge} (SB) generative models~\citep{debortoli2021diffusion, shi2024diffusion, gushchin2024asbm, tong2024simulation} provide a recent alternative to diffusion and flow priors.
An SB connects a source $X_0\sim q_0$ to a target $X_1\sim q_1$; among all path laws with these source and target marginals, it selects the one that minimizes the Kullback--Leibler (KL) divergence to a prescribed reference law. 
This combination of a general source and stochastic transport is useful for image generation~\citep{debortoli2021diffusion, chen2022likelihood, abdelrahman2025toddlerdiffusion}, image restoration~\citep{liu2023i2sb, yi2025degradation}, cell trajectory inference~\citep{tong2024simulation, fournie2026cell}, and other tasks.
Existing SB-based conditional methods~\citep{shi2022cdsb, liu2023i2sb, huang2024sb, li2025bridgesr} train a task-specific network for each likelihood family.
Training-free conditional sampling that reuses a pretrained unconditional SB prior remains unexplored.

Under a Brownian reference with diffusion scale $\sigma>0$, an SB conditioned on its endpoints $(X_0,X_1)$ is a Brownian bridge. 
At each $t\in(0,1)$, we can therefore write $X_t=(1-t)X_0+tX_1+\varepsilon_t$, where $\varepsilon_t\sim\cN(0,\sigma^2t(1-t)\I)$ is independent of $(X_0,X_1)$.
For endpoint prediction, diffusion samplers use score-based denoising, and flow samplers use velocity-based extrapolation (Table~\ref{tab:novelty}).
For SBs, however, score-only denoising estimates the interpolation $(1-t)X_0+tX_1$ instead of $X_1$, while velocity-only extrapolation omits a score correction.
Our key observation is that the probability-flow velocity $\vstar_t$ and marginal score $\sstar_t$ are complementary, and combining these two fields gives the endpoint conditional mean: $\EE[X_1 | X_t=x]=x+(1-t)\vstar_t(x)+(\sigma^2/2)(1-t)\sstar_t(x)$ (Prop.~\ref{thm:inversion}).
Their neural approximations can be obtained through score and flow matching~\citep{tong2024simulation} or from learned forward and backward SB drifts~\citep{debortoli2021diffusion, shi2024diffusion}.

\begin{table}[!ht]
  \caption{Training-free conditional sampling landscape from an endpoint prediction view.}
  \label{tab:novelty}
  \centering
  \setlength{\tabcolsep}{3pt}
  \renewcommand{\arraystretch}{1.08}
  \begin{adjustbox}{max width=\textwidth}
  \begin{tabular}{@{}lccc@{}}
    \toprule
    Model class & What is $X_t$? & Endpoint signal & Representative samplers \\
    \midrule
    Diffusion &
    noisy observation of $X_1$ &
    score (Tweedie) &
    DPS~\citep{chung2023diffusion}, TDS~\citep{wu2023practical}, FPS~\citep{dou2024diffusion} \\
    Flow matching &
    deterministic interpolation of $(X_0,X_1)$ &
    velocity (Euler) &
    FMPS~\citep{song2024fmps}, TFTF~\citep{qu2026tftf} \\
    Schr\"odinger bridge &
    noisy interpolation of $(X_0,X_1)$ &
    \textbf{velocity $+$ score} (Prop.~\ref{thm:inversion}) &
    \textbf{\ours{} (this work)} \\
    \bottomrule
  \end{tabular}
  \end{adjustbox}
\end{table}

Using these pretrained fields, we propose \ours{}, a tempered twisted SMC sampler (Figure~\ref{fig:bridgetwist-overview}).
The ideal twist, i.e., the predictive likelihood $\EE[\Ly(X_1)| X_t]$, is unavailable analytically, and using pilot rollouts to estimate it \citep{lin2010bridge, lin2013lookahead} requires many network evaluations per particle and step.
We instead use a plug-in twist $\Ly(\hat{X}_1(X_t))$ to approximate, where $\hat{X}_1(X_t)$ is the velocity--score endpoint prediction.
This twist discards uncertainty in $X_1$ given $X_t$ and can be over-confident near the source.
We therefore temper its influence with a time-varying temperature coefficient.
In an isotropic linear-Gaussian model, the tempered plug-in twist can recover the exact twist (Remark~\ref{rem:beta-star}).

With velocity and score fields, we can vary the diffusion coefficient $g(t)$ at inference time while preserving the SB marginals (Prop.~\ref{prop:fokker-planck})~\citep{tong2024simulation,albergo2025stochastic}.
We use this flexibility to introduce a bridge-shaped schedule $g(t)\propto\sqrt{t(1-t)}$, following the time dependence of the Brownian bridge standard deviation.
This schedule concentrates noise in the middle of the sampling path.
It reduces noise near the source, where guidance is weak, and near the target, where additional noise can disrupt particles already aligned with the conditioning information.

\paragraph{Our contributions.} We believe that the following progresses are novel.
\begin{itemize}
  \item \textbf{Training-free SB conditioning.}
        \ours{} is, to our knowledge, the first training-free conditional sampler for pretrained Schr\"odinger bridges, supporting arbitrary differentiable likelihoods.
  \item \textbf{Tempered plug-in twist.}
        We construct a lookahead twist from the velocity--score endpoint predictor and temper its influence along the path, without changing the final target conditional distribution.
  \item \textbf{Noise schedule and consistency.}
        We introduce a bridge-shaped diffusion schedule and importance correction; Theorem~\ref{thm:exactness} bounds the $L_2$ error for $K$ particles by order $K^{-1/2}$ relative to the learned and discretized posterior.
  \item \textbf{Plug-and-play for diverse tasks.}
        \ours{} achieves leading performance on various image generation and restoration tasks without task-specific retraining.
\end{itemize}

\section{Background}\label{sec:background}

\paragraph{Notation.}
Random objects are denoted by uppercase letters (e.g., $X_t$), and their realizations by lowercase ones (e.g., $x_t$).
We write $\cN(\mu,\Sigma)$ for a Gaussian density and $\I$ for the identity matrix.
The time index is $t\in[0,1]$, and its discretization is $t_n$ for $n\in\{0,\ldots,N\}$ with $t_0=0$ and $t_N=1$.

\subsection{Schr\"odinger bridge generative models}\label{sec:setup-sb}

Let $\Omega:=C([0,1],\RR^d)$ be the space of continuous paths in $\RR^d$, and let $q_0$ and $q_1$ be the source (or noise) and target (or data) distributions.
Let $Q$ be a Brownian reference path measure under which $\dd X_t=\sigma\,\dd W_t$, where $\sigma>0$ and $W_t$ is the standard Brownian motion.
The \emph{Schr\"odinger bridge} problem~\citep{schrodinger1932theorie, leonard2014survey} seeks the path measure on $\Omega$ closest to $Q$ in Kullback--Leibler divergence subject to the prescribed source and target marginals,
\[
P^\star=\argmin_{P:\,p_0=q_0,\,p_1=q_1}\KL(P\|Q),
\]
where $p_t$ denotes the marginal of $P$ at time $t$.
The bridge $P^\star$ admits a forward stochastic differential equation (SDE),
\begin{equation}\label{eq:sb-sde}
\dd X_t=b^\star_t(X_t)\,\dd t+\sigma\,\dd W_t,\qquad X_0\sim q_0.
\end{equation}
The associated probability-flow ordinary differential equation (ODE)~\citep{song2021scorebased, tong2024simulation} has the same marginal distributions $p_t$ as Eq.~\eqref{eq:sb-sde}:
\begin{equation}\label{eq:sb-ode}
\dd X_t=\vstar_t(X_t)\,\dd t,\qquad X_0\sim q_0.
\end{equation}
The drift $b^\star$, velocity $\vstar$, and score $\sstar$ are linked by
\[
b^\star_t(x)=\vstar_t(x)+\tfrac{1}{2}\sigma^2\sstar_t(x),
\qquad \sstar_t(x)=\nabla\log p_t(x).
\]

\paragraph{Learning the Schr\"odinger bridge.}
Earlier iterative schemes~\citep{debortoli2021diffusion, shi2024diffusion} learn the bridge through repeated forward and backward SDE simulations and updates to drift networks. This repeated SDE simulation makes them computationally expensive.
Simulation-free score and flow matching~\citep{tong2024simulation}  circumvents the need to simulate the learned SDE during training.
Instead, it draws endpoint pairs $(X_0,X_1)$ from an entropic optimal transport (OT) coupling and, for each pair, samples an intermediate state $X_t$ from the corresponding Brownian bridge.
Then the conditional bridge velocity and score at $(t,X_t)$ serve as regression targets for $(\vthe,\sthe)$, neural approximations of $(\vstar,\sstar)$. Appendix~\ref{app:pretraining} details the pretraining parameterization used in our experiments.

Given $(\vthe,\sthe)$ and a diffusion schedule $g:[0,1]\to\RR_{\geq0}$, to sample from $q_1$, one draws $X_0\sim q_0$ and simulates
\begin{equation}\label{eq:sb-prior-sde}
\dd X_t=b_{\theta,g}(t,X_t)\,\dd t+g(t)\,\dd W_t,
\qquad
b_{\theta,g}(t,x)=\vthe(t,x)+\tfrac{1}{2}g(t)^2\sthe(t,x).
\end{equation}
The choice of $g$ is an inference-time degree of freedom; under exact fields, any continuous $g\ge0$ preserves the SB marginals $p_t$~\citep{tong2024simulation, albergo2025stochastic}.
We write $p_\theta$ for the endpoint marginal of $X_1$ under learned fields.

\subsection{Sequential Monte Carlo and twisted SMC}\label{sec:setup-smc}

Sequential Monte Carlo is a Bayesian computational framework that approximates a sequence of path-space distributions $\{\pi_n\}_{n=0}^N$ using weighted particles~\citep{liu1998sequential, liu2001monte, chopin2020introduction}.
On a grid $0=t_0<\cdots<t_N=1$, we write $x_{t_{0:n}}:=(x_{t_0},\ldots,x_{t_n})$ for the discretized trajectory up to $t_n$.
The framework specifies a sequence of unnormalized targets $\{\gamma_n\}_{n=0}^N$, with normalized densities $\pi_n(x_{t_{0:n}})\propto\gamma_n(x_{t_{0:n}})$, and proposal kernels $\{q_n\}_{n=0}^N$.
At step $n=0$, $K$ particles $\{X_{t_0}^{(k)}\}_{k=1}^K$ are drawn from $q_0$ and assigned normalized weights $W_0^{(k)}\propto\gamma_0(X_{t_0}^{(k)})/q_0(X_{t_0}^{(k)})$.
For each subsequent step $n=1,\ldots,N$, the algorithm
\begin{itemize}
  \item \emph{resamples} the previous-step particles in proportion to their weights;
  \item \emph{propagates} each resampled particle through $q_n(\cdot\mid x_{t_{n-1}})$; and
  \item \emph{reweights} by the incremental weight ratio $\omega_n(x_{t_{0:n}})=\gamma_n(x_{t_{0:n}})/[\gamma_{n-1}(x_{t_{0:n-1}})q_n(x_{t_n}| x_{t_{n-1}})]$.
\end{itemize}
We write $\pi_n(\dd x_{t_n})$ for the marginal of $\pi_n$ at time $t_n$.
Under standard regularity, the weighted empirical measure at $t_N$ converges weakly to $\pi_N(\dd x_{t_N})$ as $K\to\infty$~\citep{delmoral2004feynman, chopin2020introduction}.

For conditional sampling, the path target is the prior tilted by the terminal likelihood, $\Pi^\star(\dd x_{0:1})\propto P_{\theta,g}(\dd x_{0:1})\,\Ly(x_1)$, where $P_{\theta,g}$ is the law of Eq.~\eqref{eq:sb-prior-sde}.
A naive importance sampler that draws full trajectories from $P_{\theta,g}$ and reweights by $\Ly(X_1)$ can suffer severe weight degeneracy when the conditioning is sharp.
\emph{Twisted SMC}~\citep{whiteley2014twisted, chopin2020introduction} mitigates this by inserting intermediate twisting functions, also called lookahead functions.
Under the learned prior $P_{\theta,g}$, the ideal twist is the predictive likelihood
\begin{equation}\label{eq:ideal-twist}
h_t(x):=\EE[\Ly(X_1)\mid X_t=x],\qquad t\in[0,1].
\end{equation}
This ideal twist is intractable because it integrates over all future trajectories, and direct Monte Carlo estimation is prohibitively expensive.
Existing twisted SMC samplers therefore approximate $h_t$ with model-specific surrogates, using Tweedie's endpoint mean inside the likelihood for diffusion priors~\citep{wu2023practical}, and the Euler one-step endpoint for flow priors~\citep{qu2026tftf}.
Section~\ref{sec:method} develops the counterpart for Schr\"odinger bridges.

\section{Method}\label{sec:method}

We consider conditioning information $y$ supplied through a differentiable likelihood $\Ly(x)=p(y|x)$, possibly known up to a normalizing constant.
Conditional sampling aims to draw from the endpoint posterior $\Pos(x_1)\propto p_\theta(x_1)\Ly(x_1)$.
This endpoint posterior is the terminal marginal of the path posterior $\Pi^\star$ defined in Section~\ref{sec:setup-smc}.

Section~\ref{sec:method:inversion} constructs a tractable surrogate for the ideal twist $h_t(x)=\EE[\Ly(X_1)| X_t=x]$.
Section~\ref{sec:method:algorithm} develops the \ours{} sampler by combining temporal tempering, guided proposals, a bridge-shaped diffusion schedule, and importance correction.
Section~\ref{sec:method:theory} bounds the resulting finite-particle approximation error.

\subsection{Plug-in twist from the SB endpoint predictor}\label{sec:method:inversion}

We first construct a first-moment surrogate for the ideal twist $h_t$ in~\eqref{eq:ideal-twist} from the SB velocity and score and use it as the lookahead of a tempered twisted SMC sampler.

With a Brownian reference, the SB in Eq.~\eqref{eq:sb-sde} conditioned on $(X_0,X_1)$ is a Brownian bridge, that is, $X_t|(X_0,X_1)\sim\cN(M_t,\sigma^2t(1-t)\I)$, where $M_t=(1-t)X_0+tX_1$ is the linear interpolation between the endpoints.
Writing $D=X_1-X_0$ for the endpoint displacement gives $X_1=M_t+(1-t)D$.
The conditional means of $M_t$ and $D$ given $X_t$ admit closed-form expressions in terms of the velocity and score fields and combine to yield the endpoint conditional mean.

\begin{proposition}\label{thm:inversion}
Suppose the SB marginal density $p_t$ is continuously differentiable on $\RR^d$ for every $t\in(0,1)$.
Under a Brownian-reference SB with velocity $\vstar$ and score $\sstar$ as in Eq.~\eqref{eq:sb-ode}, for every $t\in(0,1)$ and $x\in\RR^d$,
\begin{align}
\EE[M_t\mid X_t=x]
&=x+\sigma^2t(1-t)\sstar_t(x),
\label{eq:mu-M-star}\\
\EE[D\mid X_t=x]
&=\vstar_t(x)+\tfrac{1-2t}{2}\sigma^2\sstar_t(x),
\label{eq:mu-D-star}\\
\mustar{1\mid t}(x):=\EE[X_1\mid X_t=x]
&=x+(1-t)\vstar_t(x)+\tfrac{\sigma^2}{2}(1-t)\sstar_t(x).
\label{eq:mu-1-star}
\end{align}
\end{proposition}

Proposition~\ref{thm:inversion} serves two purposes:
(i) Eq.~\eqref{eq:mu-1-star} provides a closed-form endpoint predictor, that is, a linear extrapolation $\mustar{1| t}(x)=x+(1-t)b^\star_t(x)$ along the SB forward drift $b^\star$.
(ii) More importantly, identities~\eqref{eq:mu-M-star}--\eqref{eq:mu-D-star} explain why neither single channel suffices on its own: a score-only Tweedie predictor in~\eqref{eq:mu-M-star} recovers the interpolation instead of the endpoint, while a velocity-only Euler predictor $x+(1-t)\vstar_t(x)$ omits the score correction in~\eqref{eq:mu-1-star}; Section~\ref{sec:exp:ablation} ablates both.
The stochastic interpolant framework of \citet{albergo2025stochastic} provides velocity and score relations that leave the endpoint mean implicit.
Proposition~\ref{thm:inversion} explicitly recovers this mean for the Brownian bridge interpolant.
The proof is in Appendix~\ref{app:proof-inversion}.

With pretrained neural fields $\vthe$ and $\sthe$, we obtain a practical endpoint predictor
\begin{equation}\label{eq:hmu1t}
\hmu{1\mid t}(x)
:=x+(1-t)\vthe(t,x)+\tfrac{\sigma^2}{2}(1-t)\sthe(t,x),
\end{equation}
and use it to define the twist surrogate 
\[
\hh_t(x):=\Ly(\hmu{1\mid t}(x))\quad\text{for}\quad t\in(0,1),\qquad \hh_1(x)=\Ly(x).
\]

\subsection{The \ours{} sampler}\label{sec:method:algorithm}

Fix a grid $0=t_0<t_1<\cdots<t_N=1$ with $\Delta_n=t_n-t_{n-1}$, and use the Euler--Maruyama prior kernel
\begin{equation}\label{eq:prior-kernel}
M_n(x_{t_{n-1}},\dd x_{t_n})
:=\cN\big(x_{t_n};\,x_{t_{n-1}}+b_{\theta,g}(t_{n-1},x_{t_{n-1}})\Delta_n,\,g(t_{n-1})^2\Delta_n\,\I\big)\dd x_{t_n},
\end{equation}
with $b_{\theta,g}$ and $g$ defined in Eq.~\eqref{eq:sb-prior-sde}.  The corresponding path law is $P_{\theta,g}^{(n)}(\dd x_{t_{0:n}}):=q_0(\dd x_{t_0})\prod_{i=1}^n M_i(x_{t_{i-1}},\dd x_{t_i})$ for $n\le N$.

\paragraph{Tempered intermediate targets.}
Even with an exact conditional mean, the plug-in twist discards the conditional spread of $X_1 | X_t$, so in general
\[
\Ly\big(\EE[X_1\mid X_t=x]\big)
\neq \EE[\Ly(X_1)\mid X_t=x]=h_t(x).
\]
Since this spread is large at early times, evaluating the likelihood at a single endpoint prediction can make the twist over-confident (Remark~\ref{rem:beta-star}).
Therefore, we temper the plug-in twist, in the spirit of annealed importance sampling~\citep{neal2001annealed, delmoral2006sequential}, by the exponent $\beta(t)=t^a$ with $a\ge1$, so that $\beta(0)=0$ and $\beta(1)=1$.
Letting $\beta_n=\beta(t_n)$ and $\hh_n=\hh_{t_n}$, the intermediate (unnormalized) path targets are 
\[
\gamma_n^\beta(\dd x_{t_{0:n}}):=P_{\theta,g}^{(n)}(\dd x_{t_{0:n}})\,\hh_n(x_{t_n})^{\beta_n},\qquad 
\pi_n^\beta=\gamma_n^\beta/Z_n^\beta.
\] 
At $n=N$, $\beta_N=1$ and $\hh_N=\Ly$, so $\gamma_N^\beta(\dd x_{t_{0:N}})=P_{\theta,g}^{(N)}(\dd x_{t_{0:N}})\Ly(x_{t_N})$.
Let $p_\theta^{(N)}(\dd x_{t_N})$ denote the endpoint marginal of the discretized prior path law $P_{\theta,g}^{(N)}$.
Normalizing $\gamma_N^\beta$ and integrating out $x_{t_{0:N-1}}$ then gives the desired endpoint posterior $\pi_N^\beta(\dd x_{t_N})=\PosN(\dd x_{t_N})\propto p_\theta^{(N)}(\dd x_{t_N})\Ly(x_{t_N})$.

\paragraph{Guided proposal and telescoping weights.}
The continuous-time Doob $h$-transform adds $g(t)^2\nabla_x\log h_t(x)$ to the prior drift~\citep{whiteley2014twisted, wu2023practical, chopin2020introduction}.
Replacing the ideal twist $h$ by the tempered plug-in twist $\hh^\beta$, we obtain the guided proposal $q_n=\cN(m_{q,n},\Sigma_{q,n})$ with covariance $\Sigma_{q,n}=(g(t_{n-1})^2\Delta_n+\varepsilon)\I$, $\varepsilon>0$, and mean
\begin{equation}\label{eq:guided-proposal}
m_{q,n}(x_{t_{n-1}})
=x_{t_{n-1}}+\big[b_{\theta,g}(t_{n-1},x_{t_{n-1}})+w_g\,g(t_{n-1})^2\beta_{n-1}\nabla_x\log\hh_{n-1}(x_{t_{n-1}})\big]\Delta_n,
\end{equation}
where $w_g\geq 0$ is a guidance strength, and the inflation $\varepsilon\I$ makes the density ratio $M_n/q_n$ uniformly bounded in $x_{t_n}$.
The incremental weight
\begin{equation}\label{eq:incremental-weight}
\omega_n(x_{t_{n-1}},x_{t_n})
:=\frac{M_n(x_{t_{n-1}},x_{t_n})\,\hh_n(x_{t_n})^{\beta_n}}{q_n(x_{t_n}\mid x_{t_{n-1}},y)\,\hh_{n-1}(x_{t_{n-1}})^{\beta_{n-1}}}
\end{equation}
corrects the guided proposal toward $\pi_n^\beta$, and multiplying Eq.~\eqref{eq:incremental-weight} over $n$ gives the telescoping identity (Eq.~\eqref{eq:app-telescope} in Appendix~\ref{app:proof-exactness}).
All intermediate twist factors thus cancel in the telescoping product.
Therefore, although tempering changes both the intermediate targets and the proposal drift, the endpoint target is unchanged.

\begin{remark}[Tempering as variance matching]\label{rem:beta-star}
Suppose $X_1| X_t=x\sim\cN(\mustar{1| t}(x),\sigma_{1| t}^2\I)$ and $\Ly(x)=\cN(y;x,\sigma_y^2\I)$.
Then $h_t(x)=\cN(y;\mustar{1| t}(x),(\sigma_y^2+\sigma_{1| t}^2)\I)$.
For an exact conditional mean $\hmu{1| t}=\mustar{1| t}$, variance matching gives
\[
\beta^\star(t)=\frac{\sigma_y^2}{\sigma_y^2+\sigma_{1\mid t}^2},\qquad
\hh_t(x)^{\beta^\star(t)}\propto h_t(x).
\]
With $w_g=1$ and $\beta(t)=\beta^\star(t)$, the guidance in~\eqref{eq:guided-proposal} recovers the exact $h$-transform drift correction.
As $\sigma_{1| t}^2$ decreases toward the endpoint, $\beta^\star(t)$ increases to one, motivating $\beta(t)=t^a$ as a monotone surrogate that requires no variance estimate.
Appendix~\ref{app:proof-beta-star} gives the proof and the extension to observations of a subset of coordinates.
\end{remark}

\paragraph{Bridge-shaped diffusion schedule.}
At inference time, $g$ is flexible (Section~\ref{sec:setup-sb}).
A constant $g(t)\equiv\sigma$ matches the Brownian reference but allocates a uniform proposal variance $\sigma^2\Delta_n$ across $n$.
This is mismatched with the conditional bridge variance $\sigma^2t(1-t)$, which instead vanishes at both endpoints, and the excess proposal noise can increase the variance of importance weights near $t\in\{0,1\}$.
We therefore introduce a bridge-shaped schedule $g(t)=c_g\sigma\sqrt{t(1-t)+\delta_g}$, where $c_g>0$ is an inference diffusion scale and the small $\delta_g>0$ keeps $g$ positive at $t\in\{0,1\}$.

\begin{remark}[Interaction between tempering and diffusion]
The exponent $\beta(t)$ sets how much the sampler trusts the plug-in lookahead, while $g(t)$ allocates the proposal noise along the path.
Small $g(t)$ is intended to limit unguided perturbations near $t=0$, where the source already supplies particle diversity, and to avoid disturbing particles aligned with the likelihood near $t=1$.
Larger interior noise supports exploration.
Ignoring the small floor $\delta_g$, the two schedules combine in Eq.~\eqref{eq:guided-proposal} through the scalar guidance coefficient $g(t)^2\beta(t)\propto t^{a+1}(1-t)$, which peaks at $t=(a+1)/(a+2)$, or $t=2/3$ for the default $a=1$.
\end{remark}

\paragraph{Algorithm and complexity.}
Algorithm~\ref{alg:bridgetwist} summarizes the proposed \ours{} sampler, with stratified resampling triggered when the effective sample size (ESS) falls below $\rho K$.
Each step requires one forward evaluation of $(\vthe,\sthe)$ per particle and one backward pass through the endpoint predictor and the likelihood, so the total cost is $\cO(KN)$ network evaluations and $\cO(KN)$ likelihood gradients, the same order as gradient-guided samplers~\citep{chung2023diffusion, song2024fmps} and twisted SMC samplers~\citep{wu2023practical, qu2026tftf}.

\begin{algorithm}[!t]
\caption{\ours{}: tempered twisted SMC for Schr\"odinger bridges}
\label{alg:bridgetwist}
\small
\begin{algorithmic}[1]
\Require Pretrained fields $(\vthe,\sthe)$; source $q_0$; likelihood $\Ly$; grid $t_{0:N}$; diffusion schedule $g$; guidance strength $w_g$; tempering exponent $a$; number of particles $K$; ESS threshold $\rho\in(0,1]$.
\Ensure Weighted endpoint particles $\{(X_{t_N}^{(k)},W_N^{(k)})\}_{k=1}^K$ approximating $\PosN$.
\State Set $\beta_n=t_n^a$ and $b_{\theta,g}(t,x)=\vthe(t,x)+\tfrac{1}{2}g(t)^2\sthe(t,x)$.
\State Sample $X_{t_0}^{(k)}\sim q_0$ and set $\log W_0^{(k)}=\beta_0\log\hh_0(X_{t_0}^{(k)})$ for $k=1,\ldots,K$; normalize $\{W_0^{(k)}\}_{k=1}^K$.
\For{$n=1,\ldots,N$}
    \State If $\mathrm{ESS}_{n-1}:=1/\sum_k (W_{n-1}^{(k)})^2<\rho K$, resample ancestors $A_{n-1}^{(1:K)}$ and reset $\log W_{n-1}^{(k)}=-\log K$; otherwise set $A_{n-1}^{(k)}=k$. \Comment{resampling}
    \For{$k=1,\ldots,K$} \Comment{propagation and reweighting}
        \State Let $\bar X_{t_{n-1}}^{(k)}=X_{t_{n-1}}^{(A_{n-1}^{(k)})}$ and evaluate $\nabla_x\log\hh_{n-1}(\bar X_{t_{n-1}}^{(k)})$.
        \State Sample $X_{t_n}^{(k)}\sim q_n(\cdot|\bar X_{t_{n-1}}^{(k)},y)$ from Eq.~\eqref{eq:guided-proposal}.
        \State Update $\log W_n^{(k)}=\log W_{n-1}^{(k)}+\log\omega_n(\bar X_{t_{n-1}}^{(k)},X_{t_n}^{(k)})$ by Eq.~\eqref{eq:incremental-weight}.
    \EndFor
    \State Normalize $\{W_n^{(k)}\}_{k=1}^K$.
\EndFor
\State \Return $\{(X_{t_N}^{(k)},W_N^{(k)})\}_{k=1}^K$.
\end{algorithmic}
\end{algorithm}

\subsection{Finite-particle accuracy}\label{sec:method:theory}

Let $h_n(x):=\EE[\Ly(X_1)|X_{t_n}=x]$ and $P^y_{n,N}(x,\dd z):=\PP(X_1\in\dd z|X_{t_n}=x,y)$ denote the exact twist and the posterior transition kernel from step $n$ to step $N$ under the discretized prior.
Write $\osc(f):=\sup f-\inf f$ and $\varrho(P):=\sup_{\osc(f)\le1}\osc(Pf)\in[0,1]$ for the oscillation of $f$ and the contraction coefficient of a Markov kernel $P$, respectively.
The norm $\|\cdot\|_{L_2}$ denotes the root mean square over the sampler's randomness.

\begin{theorem}\label{thm:exactness}
Under the regularity conditions stated in Appendix~\ref{app:proof-exactness}, for every bounded measurable $\varphi:\RR^d\to\RR$ and every $K\ge1$, the weighted empirical measure produced by Algorithm~\ref{alg:bridgetwist} with stratified resampling satisfies
\begin{equation}\label{eq:finite-k}
\Big\|\sum_{k=1}^K W_N^{(k)}\varphi(X_{t_N}^{(k)})-\EE_{\PosN}[\varphi]\Big\|_{L_2}
\le\frac{C_N}{\sqrt K}\,\osc(\varphi),
\quad
C_N=2\sum_{n=0}^N\alpha_n\,R_{n,N}^2\,\varrho(P^y_{n,N}),
\end{equation}
where $\alpha_n<\infty$ is independent of $K$ and controls the one-step sampling error, and $R_{n,N}:=\sup_x\{h_n(x)/\hh_n(x)^{\beta_n}\}/\inf_x\{h_n(x)/\hh_n(x)^{\beta_n}\}$ measures the relative twist mismatch.
\end{theorem}

The endpoint predictor and tempering directly influence $R_{n,N}$, which equals one when $\hh_n^{\beta_n}\propto h_n$; Remark~\ref{rem:beta-star} gives the corresponding Gaussian example.
The local error constant $\alpha_n$ reflects the guided proposal, diffusion schedule, and resampling threshold $\rho$, while $\varrho(P^y_{n,N})$ is the contraction coefficient of $P^y_{n,N}$.
For a fixed grid, Eq.~\eqref{eq:finite-k} gives $L_2$ convergence of the weighted estimates to expectations under $\PosN$, the posterior induced by the learned and discretized prior, at the rate of $K^{-1/2}$.

\section{Experiments}\label{sec:experiments}

We compare \ours{} with flow matching posterior sampling (FMPS)~\citep{song2024fmps}, training-free targeted flow (TFTF)~\citep{qu2026tftf}, twisted diffusion sampler (TDS)~\citep{wu2023practical}, diffusion posterior sampling~\citep{chung2023diffusion} adapted to the SB prior (DPS-SB), and terminal importance sampling (IS).
MNIST provides a sanity check~\citep{lecun1998gradient}; CIFAR-10 and $256\times256$ CelebA-HQ provide the main image benchmarks~\citep{krizhevsky2009learning,karras2018progressive}.
Pretraining and implementation details are in Appendices~\ref{app:pretraining}--\ref{app:implementation}; additional results are in Appendix~\ref{app:additional}.
The code is publicly available at \url{https://github.com/Mengyu8042/BridgeTwist}.

\subsection{Toy example}\label{sec:exp:toy}

We first evaluate \ours{} on a two-dimensional toy example~\citep{qu2026tftf}.
The target distribution is the Gaussian mixture $q_1=0.75\,\cN((-\sqrt{2},-\sqrt{2}),\I/4)+0.25\,\cN((\sqrt{2},\sqrt{2}),\I/4)$, the source is $q_0=\cN(0,\I)$, and the Brownian reference has $\sigma=1$.
We condition on $y=(0,0)$ through the Gaussian likelihood $\Ly(x_1)=\cN(y;x_1,\I/4)$.
In this setting, the target posterior is available analytically as $\Pos(x_1)=0.75\,\cN(x_1;(-\sqrt{2}/2,-\sqrt{2}/2),\I/8)+0.25\,\cN(x_1;(\sqrt{2}/2,\sqrt{2}/2),\I/8)$.
All velocity and score fields are computed in closed form, so this experiment isolates the conditional sampler from neural network approximation errors (Appendix~\ref{app:toy-benchmark}).

We use $N=100$ discretization steps and $K\in\{100,200,500,1000,2000\}$ particles, and report the maximum mean discrepancy (MMD) to the analytic posterior over $50$ repetitions.
The left panel of Figure~\ref{fig:toy} shows that \ours{} achieves the lowest MMD.
Its empirical MMD$^2$ decays approximately as $K^{-1}$, with a fitted log-log slope of $-1.01$.
SMC-based methods improve as $K$ increases, with \ours{} outperforming both flow-based TFTF and diffusion-based TDS.
FMPS and DPS-SB plateau at higher errors because they use approximate guidance without importance weight correction, leaving bias that persists as $K$ increases.
The right panel indicates that \ours{} closely matches the two posterior modes and their relative mass.

\begin{figure}[!htbp]
  \centering
  \includegraphics[height=0.155\textheight]{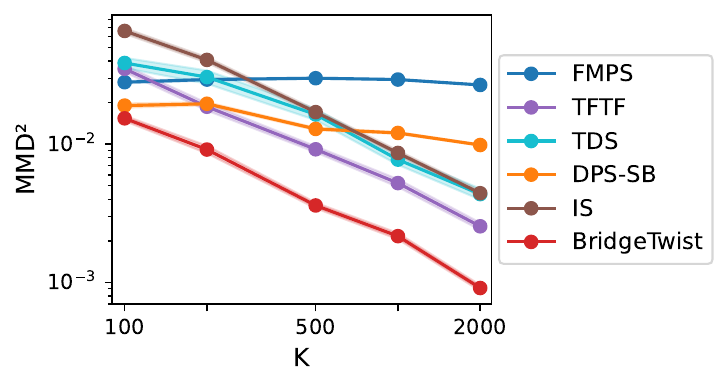}\hfill
  \includegraphics[height=0.155\textheight]{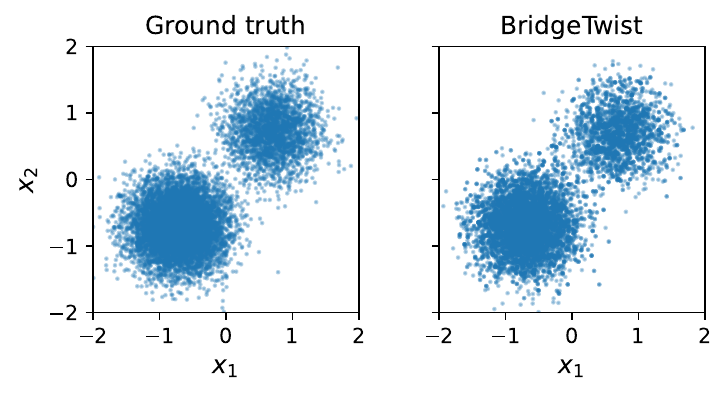}
  \caption{Toy benchmark.
Left: MMD$^2$ to the analytic posterior as the number of particles $K$ increases.
Right: ground-truth posterior samples and \ours{} samples, with $10^4$ samples in each panel.}
  \label{fig:toy}
\end{figure}

\subsection{Image conditional sampling}\label{sec:exp:image}

For a fair comparison, all methods use uniform grids with $N=800$ steps.
\ours{}, DPS-SB, and IS share the Schr\"odinger bridge prior; FMPS and TFTF share the flow prior; and TDS uses its native DDPM (MNIST) and iDDPM (CIFAR-10) priors.

\paragraph{Class-conditional sampling.}
A frozen classifier defines $\Ly(x_1)=p_\psi(y| x_1)$ on MNIST and CIFAR-10.
We measure accuracy with the likelihood classifier (ACC$_L$) and an independent external classifier (ACC$_E$).
Kernel inception distance (KID)~\citep{binkowski2018demystifying}, Density, and Coverage~\citep{naeem2020reliable} evaluate sample quality and coverage against real images within each class.
Table~\ref{tab:class-cond-k64} shows that \ours{} combines high external accuracy with low KID and broad coverage; Figure~\ref{fig:class-cond-cifar-air} illustrates class alignment and diversity.

\begin{table}[!htbp]
  \caption{Class-conditional metrics at $K=64$.
Dens./Cov.\ denote Density/Coverage.
Results are averaged over 100 random seeds. Bold marks the best value in each column.}
  \label{tab:class-cond-k64}
  \centering

  \setlength{\tabcolsep}{2.5pt}
  \renewcommand{\arraystretch}{0.96}
  \begin{adjustbox}{max width=\textwidth}
  \begin{tabular}{@{}lccccc@{\hspace{9pt}}ccccc@{}}
    \toprule
    & \multicolumn{5}{c}{MNIST} & \multicolumn{5}{c}{CIFAR-10} \\
    \cmidrule(lr){2-6}\cmidrule(l){7-11}
    Method & ACC$_L$ $\uparrow$ & ACC$_E$ $\uparrow$ & KID $\downarrow$ & Dens.\ $\uparrow$ & Cov.\ $\uparrow$ & ACC$_L$ $\uparrow$ & ACC$_E$ $\uparrow$ & KID $\downarrow$ & Dens.\ $\uparrow$ & Cov.\ $\uparrow$ \\
    \midrule
    FMPS & $0.997$ & $0.453$ & $0.0202$ & 0.026 & 0.007 & $0.899$ & $0.514$ & $0.1807$ & 0.504 & 0.390 \\
    TFTF & $0.977$ & $0.910$ & $0.0018$ & 0.954 & 0.534 & $0.986$ & $0.934$ & $0.0572$ & 1.161 & \textbf{0.600} \\
    TDS & $0.992$ & $0.942$ & $0.0015$ & 0.923 & 0.415 & $0.947$ & $0.869$ & $0.129$ & 1.159 & 0.342 \\
    DPS-SB & $\mathbf{0.999}$ & $0.619$ & $0.0010$ & 0.633 & 0.495 & $0.978$ & $0.585$ & $0.1051$ & 0.410 & 0.302 \\
    IS & $0.997$ & $0.920$ & $0.0007$ & 1.001 & 0.146 & $0.971$ & $0.865$ & $0.0744$ & 0.755 & 0.113 \\
    \ours{} & $\mathbf{0.999}$ & $\mathbf{0.952}$ & $\mathbf{0.0003}$ & \textbf{1.154} & \textbf{0.617} & $\mathbf{0.997}$ & $\mathbf{0.979}$ & $\mathbf{0.0403}$ & \textbf{1.288} & 0.596 \\
    \bottomrule
  \end{tabular}
  \end{adjustbox}
\end{table}

\begin{figure}[!t]
  \centering
  \begin{minipage}[t]{0.19\linewidth}
    \centering
    \scriptsize FMPS\\[-0.1pt]
    \includegraphics[width=\linewidth]{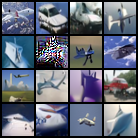}
  \end{minipage}\hfill
  \begin{minipage}[t]{0.19\linewidth}
    \centering
    \scriptsize TFTF\\[-0.1pt]
    \includegraphics[width=\linewidth]{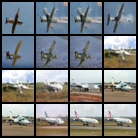}
  \end{minipage}\hfill
  \begin{minipage}[t]{0.19\linewidth}
    \centering
    \scriptsize DPS-SB\\[-0.1pt]
    \includegraphics[width=\linewidth]{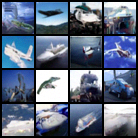}
  \end{minipage}\hfill
  \begin{minipage}[t]{0.19\linewidth}
    \centering
    \scriptsize IS\\[-0.1pt]
    \includegraphics[width=\linewidth]{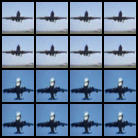}
  \end{minipage}\hfill
  \begin{minipage}[t]{0.19\linewidth}
    \centering
    \scriptsize \ours{}\\[-0.1pt]
    \includegraphics[width=\linewidth]{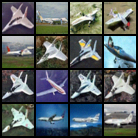}
  \end{minipage}
  \caption{CIFAR-10 airplane samples at $K=16$, after resampling based on the terminal weights.}
  \label{fig:class-cond-cifar-air}
\end{figure}

Table~\ref{tab:matched-nfe} compares methods at equal total numbers of backbone function evaluations (NFEs), counting each forward or backward pass once per particle.
Baselines adjust $K$ to fit the NFE budgets $\{B_1,B_2,B_3,B_4\}$, defined by \ours{} at $K\in\{4,16,64,256\}$.
FMPS leads at the smallest budget, while \ours{} performs the best at all the larger budgets.
Appendix~\ref{app:cost} compares runtime and memory usage.

\begin{table}[!htbp]
  \caption{CIFAR-10 ACC$_E$ at matched total NFEs.
Subscripts give standard deviations across repetitions.
Bold marks the best value.}
  \label{tab:matched-nfe}
  \centering

  \setlength{\tabcolsep}{6pt}
  \renewcommand{\arraystretch}{0.96}
  \begin{adjustbox}{max width=\textwidth}
  \begin{tabular}{@{}lcccc@{}}
    \toprule
    Method $\backslash$ total NFE & $B_1=6{,}408$ & $B_2=25{,}632$ & $B_3=102{,}528$ & $B_4=410{,}112$ \\
    \midrule
    FMPS & $\mathbf{0.775}_{\pm0.016}$ & $0.813_{\pm0.015}$ & $0.814_{\pm0.009}$ & $0.814_{\pm0.004}$ \\
    TFTF & $0.555_{\pm0.145}$ & $0.770_{\pm0.116}$ & $0.936_{\pm0.054}$ & $0.961_{\pm0.032}$ \\
    DPS-SB & $0.613_{\pm0.064}$ & $0.606_{\pm0.034}$ & $0.601_{\pm0.025}$ & $0.586_{\pm0.008}$ \\
    IS & $0.568_{\pm0.058}$ & $0.815_{\pm0.059}$ & $0.866_{\pm0.021}$ & $0.885_{\pm0.012}$ \\
    TDS & $0.585_{\pm0.120}$ & $0.638_{\pm0.098}$ & $0.875_{\pm0.113}$ & $0.918_{\pm0.059}$ \\
    \ours{} & $0.630_{\pm0.112}$ & $\mathbf{0.914}_{\pm0.061}$ & $\mathbf{0.979}_{\pm0.009}$ & $\mathbf{0.984}_{\pm0.006}$ \\
    \bottomrule
  \end{tabular}
  \end{adjustbox}

\end{table}

\paragraph{Text-prompted image generation.}
On CelebA-HQ, text prompts specify facial attributes, and a CLIP energy defines the text likelihood $\Ly(x_1)\propto\exp\{\lambda\,\operatorname{sim}(\phi_I(x_1),\phi_T(y))\}$~\citep{radford2021clip}.
We evaluate the CLIP similarity under the likelihood model (CLIP$_L$), all-attributes accuracy from an independent classifier (ACC$_E$), and pairwise LPIPS diversity within each prompt~\citep{zhang2018unreasonable,zhu2017toward}.
Table~\ref{tab:text2img-k25} shows that \ours{} attains the best prompt-image alignment and highest diversity simultaneously.
Figure~\ref{fig:text2img-ours} shows representative samples from \ours{}.

\begin{figure}[!htbp]
\begin{minipage}[t]{0.545\textwidth}
  \vspace{0.0pt}
  \centering
  \captionof{table}{Text-prompted image generation on CelebA-HQ at $K=25$.
Subscripts give standard deviations; bold marks the best value.}
  \label{tab:text2img-k25}

  \setlength{\tabcolsep}{3pt}
  \renewcommand{\arraystretch}{0.96}
  \begin{adjustbox}{max width=\textwidth}
  \begin{tabular}{@{}lccc@{}}
    \toprule
    Method & CLIP$_L$ $\uparrow$ & ACC$_E$ $\uparrow$ & LPIPS div.\ $\uparrow$ \\
    \midrule
    FMPS    & $0.307_{\pm0.001}$ & $0.689_{\pm0.020}$ & $0.012_{\pm0.003}$ \\
    TFTF    & $0.296_{\pm0.002}$ & $0.803_{\pm0.042}$ & $0.501_{\pm0.032}$ \\
    DPS-SB  & $0.306_{\pm0.002}$ & $0.794_{\pm0.024}$ & $0.599_{\pm0.054}$ \\
    \ours{} & $\mathbf{0.311}_{\pm0.002}$ & $\mathbf{0.827}_{\pm0.044}$ & $\mathbf{0.613}_{\pm0.048}$ \\
    \bottomrule
  \end{tabular}
  \end{adjustbox}
\end{minipage}\hfill
\begin{minipage}[t]{0.425\textwidth}
  \vspace{0pt}
  \centering
  \captionsetup{justification=centering,position=auto,aboveskip=0pt,belowskip=7pt}
  \begin{minipage}[t]{0.30\linewidth}
    \centering
    \includegraphics[width=\linewidth]{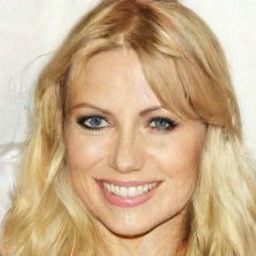}
    \captionof*{figure}{\scriptsize ``a smiling woman with blonde hair and bangs''}
  \end{minipage}\hfill
  \begin{minipage}[t]{0.30\linewidth}
    \centering
    \includegraphics[width=\linewidth]{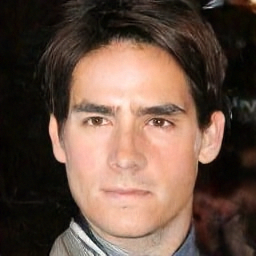}
    \captionof*{figure}{\scriptsize ``a man with black hair and an exposed forehead''}
  \end{minipage}\hfill
  \begin{minipage}[t]{0.30\linewidth}
    \centering
    \includegraphics[width=\linewidth]{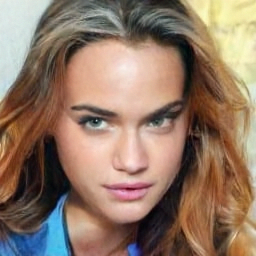}
    \captionof*{figure}{\scriptsize ``a woman with a serious expression''}
  \end{minipage}
  \vspace{-0.5em}
  \captionof{figure}{Text-prompted samples.}
  \label{fig:text2img-ours}
\end{minipage}
\end{figure}

\paragraph{Super-resolution, deblurring, and inpainting.}
We further compare with recent image restoration solvers, including PnP-Flow~\citep{martin2025pnpflow}, Flower~\citep{pourya2026flower}, and FlowDPS~\citep{kim2025flowdps}, on CelebA-HQ $4\times$ super-resolution and Gaussian deblurring, and on CIFAR-10 inpainting with half-image and random 50\% masks.
\ours{} achieves the highest peak signal-to-noise ratio (PSNR) and structural similarity (SSIM) on deblurring and ranks a close second to PnP-Flow on super-resolution (Table~\ref{tab:solvers}).
On CIFAR-10, it also achieves the highest SSIM under both inpainting masks.
These results show that \ours{} is competitive with specialized image restoration solvers while supporting more general conditioning.
See Appendices~\ref{app:inpainting-mean} and~\ref{app:expanded-inpainting} for additional results.

\begin{table}[!htbp]
  \caption{Image restoration on CelebA-HQ and CIFAR-10.
Subscripts give standard deviations; dashes indicate unavailable results; bold marks the best value.}
  \label{tab:solvers}
  \centering

  \setlength{\tabcolsep}{4pt}
  \renewcommand{\arraystretch}{0.96}
  \begin{adjustbox}{max width=\textwidth}
  \begin{tabular}{@{}lcccccc@{}}
    \toprule
    & \multicolumn{4}{c}{CelebA-HQ} & \multicolumn{2}{c}{CIFAR-10} \\
    \cmidrule(lr){2-5}\cmidrule(l){6-7}
    & \multicolumn{2}{c}{Super-resolution} & \multicolumn{2}{c}{Deblurring} & \multicolumn{2}{c}{Inpainting} \\
    \cmidrule(lr){2-3}\cmidrule(lr){4-5}\cmidrule(l){6-7}
    Method & PSNR $\uparrow$ & SSIM $\uparrow$ & PSNR $\uparrow$ & SSIM $\uparrow$ & SSIM (half) $\uparrow$ & SSIM (random) $\uparrow$ \\
    \midrule
    PnP-Flow & $\mathbf{28.47}_{\pm1.05}$ & $\mathbf{0.843}_{\pm0.020}$ & $30.64_{\pm1.08}$ & $0.818_{\pm0.019}$ & $0.623_{\pm0.104}$ & $0.940_{\pm0.109}$ \\
    Flower & $27.88_{\pm0.99}$ & $0.827_{\pm0.019}$ & $30.91_{\pm1.07}$ & $0.839_{\pm0.015}$ & -- & -- \\
    FlowDPS & $26.43_{\pm1.86}$ & $0.795_{\pm0.025}$ & $29.12_{\pm0.94}$ & $0.819_{\pm0.028}$ & -- & -- \\
    FMPS & $23.68_{\pm1.42}$ & $0.622_{\pm0.042}$ & $17.92_{\pm1.67}$ & $0.689_{\pm0.038}$ & $0.240_{\pm0.114}$ & $0.554_{\pm0.106}$ \\
    TFTF & $23.78_{\pm0.95}$ & $0.772_{\pm0.032}$ & $27.98_{\pm0.83}$ & $0.737_{\pm0.032}$ & $0.532_{\pm0.071}$ & $0.875_{\pm0.054}$ \\
    DPS-SB & $28.06_{\pm0.97}$ & $0.791_{\pm0.020}$ & $29.49_{\pm0.82}$ & $0.708_{\pm0.012}$ & $0.518_{\pm0.086}$ & $0.782_{\pm0.082}$ \\
    \ours{} & $28.33_{\pm0.88}$ & $0.836_{\pm0.024}$ & $\mathbf{31.13}_{\pm0.97}$ & $\mathbf{0.847}_{\pm0.020}$ & $\mathbf{0.655}_{\pm0.072}$ & $\mathbf{0.943}_{\pm0.087}$ \\
    \bottomrule
  \end{tabular}
  \end{adjustbox}
\end{table}


\subsection{Ablation studies}\label{sec:exp:ablation}

\noindent
\begin{minipage}[t]{0.51\linewidth}
\vspace{0pt}
Table~\ref{tab:ablations} compares endpoint predictor channels under the same SB prior on CIFAR-10.
Replacing the full predictor with either the velocity-only Euler predictor or the score-only Tweedie predictor from Section~\ref{sec:method:inversion} substantially degrades both class alignment and sample quality.
See Appendices~\ref{app:calibration} and~\ref{app:ess} for more diagnostics.
\end{minipage}\hfill
\begin{minipage}[t]{0.46\linewidth}\captionsetup{hypcap=false}
\captionsetup[table]{skip=0pt}
  \captionof{table}{CIFAR-10 predictor ablation at $K=16$.
Bold marks the best value in each column.}
  \label{tab:ablations}
  \centering
  \small
  \setlength{\tabcolsep}{3pt}
  \renewcommand{\arraystretch}{0.96}
\begin{adjustbox}{max width=\textwidth}
\begin{tabular}{@{}lccc@{}}
  \toprule
  Predictor     & ACC$_L$ $\uparrow$ & ACC$_E$ $\uparrow$ & KID $\downarrow$ \\
  \midrule
  Velocity-only & $0.526$           & $0.485$            & $0.2805$ \\
  Score-only    & $0.711$           & $0.514$            & $0.2339$ \\
  \ours{}       & $\mathbf{0.948}$  & $\mathbf{0.914}$   & $\mathbf{0.0873}$ \\
  \bottomrule
\end{tabular}
\end{adjustbox}
\end{minipage}
\par\medskip

Table~\ref{tab:ablation-components-full} isolates the $h$-transform guidance, tempering, and importance weight correction under the same SB prior on MNIST and CIFAR-10.
IS weights and resamples particles only at the endpoint, whereas DPS-SB uses guidance without importance correction.
On CIFAR-10, both yield lower ACC$_E$ and higher KID than \ours{}, as do variants without $h$-transform guidance or tempering.
On MNIST, removing tempering slightly increases ACC$_E$ while worsening KID.

\begin{table}[!htbp]
  \caption{Sampler omponent ablation at $K=64$ under the same SB prior.
Bold marks the best value in each column.}
  \label{tab:ablation-components-full}
  \centering

  \setlength{\tabcolsep}{2.5pt}
  \renewcommand{\arraystretch}{0.96}
  \begin{adjustbox}{max width=\textwidth}
  \begin{tabular}{@{}lccccccc@{}}
    \toprule
    & & & & \multicolumn{2}{c}{MNIST} & \multicolumn{2}{c}{CIFAR-10} \\
    \cmidrule(lr){5-6}\cmidrule(l){7-8}
    Variant & $h$-transform & tempering & weight corr. & ACC$_E$ $\uparrow$ & KID $\downarrow$ & ACC$_E$ $\uparrow$ & KID $\downarrow$ \\
    \midrule
    IS & $\times$ & $\times$ & endpoint only & 0.920 & 0.0007 & 0.865 & 0.0744 \\
    DPS-SB & $\checkmark$ & $\times$ & $\times$ & 0.619 & 0.0010 & 0.585 & 0.1051 \\
    \ours{} w/o $h$-transform & $\times$ & $\checkmark$ & $\checkmark$ & 0.932 & \textbf{0.0003} & 0.925 & 0.0437 \\
    \ours{} w/o tempering & $\checkmark$ & $\times$ & $\checkmark$ & \textbf{0.958} & 0.0004 & 0.938 & 0.0813 \\
    \ours{} & $\checkmark$ & $\checkmark$ & $\checkmark$ & 0.952 & \textbf{0.0003} & \textbf{0.979} & \textbf{0.0403} \\
    \bottomrule
  \end{tabular}
  \end{adjustbox}
\end{table}

Table~\ref{tab:ablation-hyper} varies the tempering exponent $a\in\{0,1,2,3\}$ in $\beta(t)=t^a$ and compares bridge-shaped, constant, and cosine diffusion schedules $g(t)$ on CIFAR-10.
The default combination of linear tempering ($a=1$) and the bridge-shaped schedule performs the best.
Tempering is important, since removing it ($a=0$) largely degrades performance, whereas \ours{} is insensitive to the specific choice among $a\in\{1,2,3\}$
For each $a$, our bridge-shaped design outperforms the common constant and cosine schedules.
Appendix~\ref{app:tempering-transfer} evaluates tempering in TDS and TFTF.

\begin{table}[!htbp]
  \caption{Joint tempering and diffusion ablation on CIFAR-10 at $K=64$.
Each entry reports ACC$_E$ ($\uparrow$) / KID ($\downarrow$).
Bold marks the best values.}
  \label{tab:ablation-hyper}
  \centering

  \setlength{\tabcolsep}{6pt}
  \renewcommand{\arraystretch}{0.96}
  \begin{adjustbox}{max width=\textwidth}
  \begin{tabular}{@{}lccc@{}}
    \toprule
    $a$ $\backslash$ $g(t)$ & bridge-shaped (default) & constant ($c_g\sigma$) & cosine ($c_g\sigma\cos(\pi t/2)$) \\
    \midrule
    $a=0$ (no tempering) & 0.938 / 0.0813 & 0.932 / 0.1489 & 0.897 / 0.0920 \\
    $a=1$ (default) & $\mathbf{0.979}$ / $\mathbf{0.0403}$ & 0.967 / 0.1099 & 0.945 / 0.0534 \\
    $a=2$ & 0.971 / 0.0426 & 0.970 / 0.1034 & 0.916 / 0.0487 \\
    $a=3$ & 0.967 / 0.0470 & 0.946 / 0.1013 & 0.909 / 0.0493 \\
    \bottomrule
  \end{tabular}
  \end{adjustbox}
\end{table}

In summary, these ablations under the same SB prior support attributing the gains to the velocity--score endpoint predictor, $h$-transform guidance, tempering, and bridge-shaped diffusion.
Appendix~\ref{app:pretrained-models} further shows the gains of \ours{} do
not come from a better prior model. 

\label{sec:experiments:end}

\section{Discussion}\label{sec:discussion}

We introduce \ours{}, a training-free conditional sampler that reuses pretrained unconditional Schr\"odinger bridge models.
Central to our approach are three contributions to SB conditioning: a plug-in lookahead built from the velocity--score endpoint predictor, temporal tempering that moderates this lookahead when conditional endpoint uncertainty is large, and a bridge-shaped diffusion schedule that concentrates exploration in the path interior.
These designs guide and reweight particles through a sequence of intermediate targets toward the endpoint posterior.
We establish a finite-particle error bound for estimates under the posterior, and demonstrate competitive performance across class-conditional and text-prompted image generation, and image restoration.

Beyond the sampler itself, the endpoint predictor offers a two-channel view that connects the transport perspective of flow models with the denoising perspective of diffusion models.
For an SB under a Brownian reference, the velocity channel describes transport along the bridge, while the score channel corrects for noise in the current state.
Our ablations support combining both channels to improve class alignment and sample quality.

\paragraph{Limitations and future work.}
First, \ours{} currently is restricted to Schr\"odinger bridges with the Brownian reference, and extending it to other references require reference-specific endpoint predictors.
Second, its twist approximation uses only the conditional endpoint mean and can be inaccurate for strongly curved likelihoods or multimodal conditional endpoint laws.
The second-order extension in Appendix~\ref{app:second-order-twist} improves some low-dimensional cases, but full covariance estimation is expensive in high dimensions and requires scalable approximations.
Third, the sampler targets the posterior induced by the learned and discretized prior, leaving prior-learning and discretization errors.
Fourth, preparing the SB prior also requires an additional score head, with a one-time computational overhead relative to directly using pretrained flow or diffusion models.
Finally, future work also includes protein design and broader single-cell transport applications, building on the initial example in Appendix~\ref{app:scrna}.

\section*{Acknowledgments and Disclosures}
The authors thank the anonymous referees for their insightful and constructive comments, which have led to an improved version of the paper.
This work was supported by the China Postdoctoral Science Foundation (Grant No. 2025M783111) and the Shenzhen Science and Technology Program (Grant No. AI2026019).
The authors declare no competing interests.

\clearpage
\bibliography{refs}
\bibliographystyle{unsrtnat}


\clearpage
\appendix
\addtocontents{toc}{\protect\setcounter{tocdepth}{2}}
\begingroup

\renewcommand{\contentsname}{Appendix contents}
\pdfbookmark[0]{Appendix contents}{appendix-contents}
\tableofcontents
\endgroup
\clearpage
\section{Proofs}\label{app:proofs}

In this appendix, we prove main theoretical results. Appendix~\ref{app:proof-inversion} proves the endpoint decomposition in Proposition~\ref{thm:inversion}, Appendix~\ref{app:proof-fokker-planck} proves the marginal preservation property of the inference SDE, and Appendix~\ref{app:proof-exactness} proves the finite-particle bound of Theorem~\ref{thm:exactness}; Appendix~\ref{app:proof-beta-star} proves Remark~\ref{rem:beta-star}.

We work in the Brownian-reference Schr\"odinger bridge setup of Section~\ref{sec:setup-sb}: conditional on $(X_0,X_1)$,
\[
X_t \mid (X_0,X_1) \sim \cN(M_t,\TT_t^2\I),
\qquad
M_t=(1-t)X_0+tX_1,
\qquad
\TT_t^2=\sigma^2t(1-t).
\]
We also write $D=X_1-X_0$.
Throughout, $p_t$ denotes the SB marginal density and is assumed to be continuously differentiable on $\RR^d$ for $t\in(0,1)$.

\subsection{Proof of Proposition~\ref{thm:inversion}}\label{app:proof-inversion}

\begin{proof}
Fix $t\in(0,1)$ and $x\in\RR^d$.
The Brownian reference SB is a mixture of Brownian bridges over its endpoint coupling.
Conditional on $X_t=x$ and $X_1=x_1$, the future path is therefore a Brownian bridge from $(t,x)$ to $(1,x_1)$, satisfying
\[
\dd X_s=\frac{x_1-X_s}{1-s}\,\dd s+\sigma\,\dd\widetilde W_s,
\qquad t\le s<1,
\]
where $\widetilde W$ is Brownian motion under the conditional law.
Averaging this drift at time $t$ over $X_1|X_t=x$ gives
\[
b^\star_t(x)
=\EE\Big[\frac{X_1-X_t}{1-t}\,\Big|\,X_t=x\Big]
=\frac{\EE[X_1\mid X_t=x]-x}{1-t}.
\]
Using $b^\star_t=\vstar_t+\sigma^2\sstar_t/2$ and rearranging gives
\begin{equation}\label{eq:app-mu-1}
\EE[X_1\mid X_t=x]
=x+(1-t)\vstar_t(x)+\frac{\sigma^2}{2}(1-t)\sstar_t(x).
\end{equation}

To recover $\EE[M_t| X_t=x]$, let $\nu_t$ be the law of $M_t$ and $\phi_{\TT_t}$ the $\cN(0,\TT_t^2\I)$ density.
Since $X_t|M_t\sim\cN(M_t,\TT_t^2\I)$, the marginal density of $X_t$ is the convolution
\begin{equation}\label{eq:app-convolution}
p_t(x)=\int\phi_{\TT_t}(x-m)\,\nu_t(\dd m).
\end{equation}
Differentiating Eq.~\eqref{eq:app-convolution} with respect to $x$ gives
\begin{align}
\nabla_x p_t(x)
&=\int\nabla_x\phi_{\TT_t}(x-m)\,\nu_t(\dd m)\notag\\
&=-\frac{1}{\TT_t^2}\int(x-m)\phi_{\TT_t}(x-m)\,\nu_t(\dd m)\notag\\
&=-\frac{1}{\TT_t^2}\bigg[xp_t(x)-\int m\phi_{\TT_t}(x-m)\,\nu_t(\dd m)\bigg].
\label{eq:app-density-gradient}
\end{align}
By Bayes' rule, the conditional law of $M_t$ given $X_t=x$ is $\phi_{\TT_t}(x-m)\nu_t(\dd m)/p_t(x)$, so
\[
\int m\phi_{\TT_t}(x-m)\,\nu_t(\dd m)
=p_t(x)\EE[M_t\mid X_t=x].
\]
Dividing Eq.~\eqref{eq:app-density-gradient} by $p_t(x)$ and rearranging yields Tweedie's identity for the interpolation $M_t$,
\begin{equation}\label{eq:app-mu-M}
\EE[M_t\mid X_t=x]
=x+\TT_t^2\nabla_x\log p_t(x)
=x+\sigma^2t(1-t)\sstar_t(x).
\end{equation}

Finally, $X_1=M_t+(1-t)D$ and Eqs.~\eqref{eq:app-mu-1} and~\eqref{eq:app-mu-M} yield
\begin{equation}\label{eq:app-mu-D}
\EE[D\mid X_t=x]
=\frac{\EE[X_1\mid X_t=x]-\EE[M_t\mid X_t=x]}{1-t}
=\vstar_t(x)+\frac{1-2t}{2}\sigma^2\sstar_t(x).
\end{equation}
\end{proof}

\begin{corollary}[Source-side conditional mean]\label{cor:source}
Under the same assumptions, for every $t\in(0,1)$ and $x\in\RR^d$,
\[
\EE[X_0\mid X_t=x]
=x-t\,\vstar_t(x)+\frac{\sigma^2}{2}\,t\,\sstar_t(x).
\]
\end{corollary}

\begin{proof}
The source state satisfies $X_0=M_t-tD$.
Combining \eqref{eq:app-mu-M} and \eqref{eq:app-mu-D},
\begin{align*}
\EE[X_0\mid X_t=x]
&=\EE[M_t\mid X_t=x]-t\,\EE[D\mid X_t=x]\\
&=x+\sigma^2t(1-t)\sstar_t(x)
-t\vstar_t(x)
-t\frac{1-2t}{2}\sigma^2\sstar_t(x)\\
&=x-t\vstar_t(x)
+\sigma^2t\left[(1-t)-\frac{1-2t}{2}\right]\sstar_t(x)\\
&=x-t\vstar_t(x)+\frac{\sigma^2}{2}t\sstar_t(x).
\end{align*}
\end{proof}

\subsection{Inference SDE preserves the SB marginal}\label{app:proof-fokker-planck}

\begin{proposition}\label{prop:fokker-planck}
Let $g\in C([0,1],\RR_{\geq 0})$ be an inference diffusion schedule.
The Brownian-reference Schrodinger bridge SDE~\eqref{eq:sb-sde} is
\begin{equation}\label{eq:app-sb-original-sde}
\dd \bar X_t=b^\star_t(\bar X_t)\,\dd t+\sigma\,\dd W_t,
\quad
b^\star_t(x)=\vstar_t(x)+\frac{1}{2}\sigma^2\sstar_t(x),
\quad
\bar X_0\sim q_0,
\end{equation}
with $\sstar_t(x)=\nabla\log p_t(x)$ and $q_0=p_0$. Define the inference-time drift
\[
b_{\theta,g}(t,x):=\vstar_t(x)+\frac{1}{2}g(t)^2\sstar_t(x),
\]
and the corresponding inference SDE
\begin{equation}\label{eq:app-inf-sde}
\dd X_t=b_{\theta,g}(t,X_t)\,\dd t+g(t)\,\dd W_t,
\qquad X_0\sim q_0.
\end{equation}
For any such $g$, the marginal density of $X_t$ under \eqref{eq:app-inf-sde} equals the SB marginal density $p_t$ of \eqref{eq:app-sb-original-sde} for every $t\in(0,1)$.
\end{proposition}

\begin{proof}
The SB probability-flow ODE associated with the marginal velocity $\vstar$ transports $p_0=q_0$ along the marginal densities $\{p_t\}_{t\in(0,1)}$.
Equivalently, these densities satisfy the continuity equation
\begin{equation}\label{eq:app-sb-continuity}
\partial_t p_t(x)
=-\nabla\cdot[p_t(x)\vstar_t(x)].
\end{equation}

Let $\tilde p_t$ denote the marginal density of the SDE \eqref{eq:app-inf-sde}.
Its Fokker--Planck equation is
\begin{equation}\label{eq:app-fp}
\partial_t\tilde p_t(x)
=-\nabla\cdot[\tilde p_t(x)b_{\theta,g}(t,x)]
+\frac{1}{2}g(t)^2\Delta\tilde p_t(x),
\end{equation}
with initial condition $\tilde p_0=q_0$.
Because the SB source constraint gives $p_0=q_0$, the two initial conditions coincide.
We now check that the SB density $p_t$ also solves \eqref{eq:app-fp}.
Substituting $p_t$ in the right-hand side of
\eqref{eq:app-fp}, and using
$b_{\theta,g}=\vstar+\frac{1}{2}g^2\sstar$, gives
\begin{align}
&-\nabla\cdot[p_t b_{\theta,g}]
+\frac{1}{2}g(t)^2\Delta p_t \notag\\
&\qquad
=-\nabla\cdot[p_t\vstar]
-\frac{1}{2}g(t)^2\nabla\cdot[p_t\sstar]
+\frac{1}{2}g(t)^2\Delta p_t.
\label{eq:app-fp-cancel-start}
\end{align}
Since $\sstar_t(x)=\nabla\log p_t(x)$,
\[
p_t(x)\sstar_t(x)=p_t(x)\nabla\log p_t(x)=\nabla p_t(x),
\]
and hence
\[
\nabla\cdot[p_t\sstar]=\nabla\cdot[\nabla p_t]=\Delta p_t.
\]
The last two terms in \eqref{eq:app-fp-cancel-start} therefore cancel exactly:
\[
-\frac{1}{2}g(t)^2\nabla\cdot[p_t\sstar]
+\frac{1}{2}g(t)^2\Delta p_t=0.
\]
Then the right-hand side of \eqref{eq:app-fp-cancel-start} reduces to $-\nabla\cdot[p_t\vstar]$,
which equals $\partial_t p_t$ by the continuity equation \eqref{eq:app-sb-continuity}.
Thus $p_t$ satisfies the same Fokker--Planck equation and the same initial condition as $\tilde p_t$, and hence $\tilde p_t=p_t$ for all $t\in(0,1)$.

This result is independent of the specific choice of $g$. Changing $g$ adds both a diffusion term and the drift term $\frac{1}{2}g(t)^2\nabla\log p_t$ with the same coefficient, so the two contributions cancel at the density level.
This is the flexible schedule property used in Section~\ref{sec:method:algorithm}.
\end{proof}

\subsection{Proof of Theorem~\ref{thm:exactness}}\label{app:proof-exactness}

Theorem~\ref{thm:exactness} uses the following regularity conditions on the tempered twist and the Gaussian proposal kernels of Algorithm~\ref{alg:bridgetwist}.

\begin{assumption}[Endpoint twist consistency]\label{ass:terminal}
For all $x\in\RR^d$, $\hh_N(x)^{\beta_N}=\Ly(x)$.
\end{assumption}

\begin{assumption}[Bounded likelihood and twist]\label{ass:bounded-twist}
There are constants $0<\ell_n\le u_n<\infty$ such that $\ell_n\le\hh_n(x)^{\beta_n}\le u_n$ for all $x\in\RR^d$ and $n\in\{0,\ldots,N\}$; in particular, $\ell_N\le\Ly(x)\le u_N$.
\end{assumption}

\begin{assumption}[Bounded log-twist gradient]\label{ass:bounded-grad}
For each $n$ with $\beta_n>0$, the map $x\mapsto \log\hh_n(x)^{\beta_n}$ is continuously differentiable and has uniformly bounded gradient.
\end{assumption}

\begin{assumption}[Proposal coverage]\label{ass:proposal-variance}
For every $n$, the prior covariance is nondegenerate, i.e., $\sigma_{M,n}^2:=g(t_{n-1})^2\Delta_n>0$, and the proposal covariance $\Sigma_{q,n}$ of Section~\ref{sec:method:algorithm} strictly dominates it:
\[
\Sigma_{M,n}=\sigma_{M,n}^2\I,
\qquad
\Sigma_{q,n}=(\sigma_{M,n}^2+\varepsilon)\I,
\qquad
\varepsilon>0.
\]
The nondegeneracy condition is satisfied by the bridge-shaped schedule with $\delta_g>0$ used in the experiments and by any strictly positive inference diffusion schedule on the discretization grid.
\end{assumption}

\paragraph{Mildness of the assumptions.}
Assumption~\ref{ass:terminal} holds by construction because $\beta_N=1$ and $\hh_N(x)=\Ly(x)$.
Assumption~\ref{ass:bounded-twist} is satisfied by bounded likelihoods with a positive floor and by CLIP likelihoods with bounded similarity scores and a finite guidance scale $\lambda$.
It also implies $\ell_N\le h_n(x)\le u_N$, so the relative twist mismatch $R_{n,N}$ in Theorem~\ref{thm:exactness} is finite.
Assumption~\ref{ass:bounded-grad} follows from bounded likelihood gradients and bounded Jacobians of $(\vthe,\sthe)$ through the chain rule applied to Eq.~\eqref{eq:hmu1t}, matching the regularity used in diffusion twisting analyses~\citep[Appendix~A.5]{wu2023practical}; the explicit chain-rule derivation is recorded below.
Assumption~\ref{ass:proposal-variance} is the standard absolute-continuity safeguard: the small $\varepsilon\I$ inflation gives the proposal heavier Gaussian tails than the prior kernel.

\begin{lemma}[Ratio estimator]\label{lem:ratio}
Let $(Z_i,F_i)_{i=1}^K$ be independent pairs with $Z_i>0$, $\EE[Z_i^2]<\infty$, and $F_i\in[m,M]$ almost surely.
For $\theta:=\sum_i\EE[Z_iF_i]/\sum_i\EE[Z_i]$ and $\hat\theta:=\sum_iZ_iF_i/\sum_iZ_i$, we have
\[
\EE[(\hat\theta-\theta)^2]\le8(M-m)^2\,\frac{\sum_i\EE[Z_i^2]}{(\sum_i\EE[Z_i])^2}.
\]
\end{lemma}
\begin{proof}
Set $D:=\sum_iZ_i$, $d:=\EE[D]$, and $Y:=\sum_iZ_i(F_i-\theta)$.
Independence gives $\EE[Y]=0$, $\EE[Y^2]\le(M-m)^2\sum_i\EE[Z_i^2]$, and $\PP(D<d/2)\le4\sum_i\EE[Z_i^2]/d^2$ by Chebyshev's inequality.
Since $\hat\theta-\theta=Y/D$ and $|\hat\theta-\theta|\le M-m$, splitting on $D\ge d/2$ yields
\[
\EE[(\hat\theta-\theta)^2]
\le\frac{4\EE[Y^2]}{d^2}+(M-m)^2\PP(D<d/2)
\le\frac{8(M-m)^2\sum_i\EE[Z_i^2]}{d^2}.
\]
\end{proof}

\begin{proof}[Proof of Theorem~\ref{thm:exactness}]
The proof has five steps and follows the standard twisted SMC theory used in prior work~\citep{wu2023practical} and related Feynman--Kac analyses~\citep[Chapters~5 and~11]{chopin2020introduction}.
We first verify the final Feynman--Kac target on path space.
We then verify that the incremental weights are uniformly bounded.
We next express Algorithm~\ref{alg:bridgetwist} as a Feynman--Kac particle system whose future kernels are the likelihood-conditioned prior dynamics, bound the sampling error introduced at one step, and propagate the one-step errors to step $n=N$.

\paragraph{Step 1: verify the Feynman--Kac target and endpoint marginal.}
Let $Q_N$ denote the proposal path law used by Algorithm~\ref{alg:bridgetwist},
\[
Q_N(\dd x_{t_{0:N}})
=q_0(\dd x_{t_0})\prod_{n=1}^N q_n(x_{t_n}\mid x_{t_{n-1}},y)\,\dd x_{t_n},
\]
where $q_0$ is the SB source distribution; the algorithm initializes particles directly from $q_0$, which is also the initial marginal of the discretized SB prior path measure $P^{(N)}_{\theta,g}(\dd x_{t_{0:N}})=q_0(\dd x_{t_0})\prod_{n=1}^N M_n(x_{t_{n-1}},\dd x_{t_n})$ defined in Section~\ref{sec:method:algorithm}.
The algorithm attaches the initial weight $\hh_0(x_{t_0})^{\beta_0}$ and then the incremental weights \eqref{eq:incremental-weight}.
Therefore the unnormalized path measure represented by the weighted particles at step $N$ is
\begin{equation}\label{eq:app-fk-from-proposal}
\Gamma_N^\beta(\dd x_{t_{0:N}})
=
Q_N(\dd x_{t_{0:N}})
\hh_0(x_{t_0})^{\beta_0}
\prod_{n=1}^N\omega_n(x_{t_{n-1}},x_{t_n}).
\end{equation}
Substituting the definition of $\omega_n$ gives the telescoping product
\begin{align}
\hh_0(x_{t_0})^{\beta_0}
\prod_{n=1}^N\omega_n(x_{t_{n-1}},x_{t_n})
&=
\hh_0(x_{t_0})^{\beta_0}
\prod_{n=1}^N
\frac{M_n(x_{t_{n-1}},x_{t_n})\hh_n(x_{t_n})^{\beta_n}}
{q_n(x_{t_n}\mid x_{t_{n-1}},y)\hh_{n-1}(x_{t_{n-1}})^{\beta_{n-1}}}
\notag\\
&=
\frac{\prod_{n=1}^N M_n(x_{t_{n-1}},x_{t_n})}
{\prod_{n=1}^N q_n(x_{t_n}\mid x_{t_{n-1}},y)}
\hh_N(x_{t_N})^{\beta_N}.
\label{eq:app-telescope}
\end{align}
All intermediate twist powers cancel pairwise, and only the endpoint twist remains.
Multiplying \eqref{eq:app-telescope} by the proposal path measure in \eqref{eq:app-fk-from-proposal} yields
\[
\Gamma_N^\beta(\dd x_{t_{0:N}})
=q_0(\dd x_{t_0})\prod_{n=1}^N M_n(x_{t_{n-1}},\dd x_{t_n})\,
\hh_N(x_{t_N})^{\beta_N}
=P_{\theta,g}^{(N)}(\dd x_{t_{0:N}})\,\hh_N(x_{t_N})^{\beta_N},
\]
where the second equality uses the factorization of $P_{\theta,g}^{(N)}$ recorded above.
By Assumption~\ref{ass:terminal}, $\hh_N(x_{t_N})^{\beta_N}=\Ly(x_{t_N})$, and hence
\begin{equation}\label{eq:app-terminal-target}
\Gamma_N^\beta(\dd x_{t_{0:N}})
=P_{\theta,g}^{(N)}(\dd x_{t_{0:N}})\,\Ly(x_{t_N}).
\end{equation}
Thus the final Feynman--Kac target is the discretized SB prior path law tilted only by the true endpoint likelihood.
The intermediate twists and the tempering schedule affect the proposal and weight variance, but not the final target.
Normalizing this path measure also identifies its endpoint marginal.
Let $Z_N^\beta=\Gamma_N^\beta(1)$ and $\pi_N^\beta=\Gamma_N^\beta/Z_N^\beta$.
For any bounded measurable $\varphi:\RR^d\to\RR$, \eqref{eq:app-terminal-target} gives
\begin{align}
\int \varphi(x_{t_N})\,\pi_N^\beta(\dd x_{t_{0:N}})
&=
\frac{
\int \varphi(x_{t_N})\,\Ly(x_{t_N})\,
P_{\theta,g}^{(N)}(\dd x_{t_{0:N}})
}{
\int \Ly(x_{t_N})\,
P_{\theta,g}^{(N)}(\dd x_{t_{0:N}})
}
\notag\\
&=
\frac{
\int \varphi(x_{t_N})\,\Ly(x_{t_N})\,
p_\theta^{(N)}(\dd x_{t_N})
}{
\int \Ly(x_{t_N})\,
p_\theta^{(N)}(\dd x_{t_N})
}
\notag\\
&=
\EE_{X_{t_N}\sim\PosN}[\varphi(X_{t_N})].
\label{eq:app-endpoint-marginal}
\end{align}
The second equality integrates out $x_{t_{0:N-1}}$ under the discretized SB prior path law, whose endpoint marginal is $p_\theta^{(N)}$.
This proves that the endpoint marginal of $\pi_N^\beta$ is exactly the posterior $\PosN$ appearing in Theorem~\ref{thm:exactness}.

\paragraph{Step 2: bound the incremental weights.}
The incremental weight factors as
\[
\omega_n(x_{t_{n-1}},x_{t_n})
=
\frac{\hh_n(x_{t_n})^{\beta_n}}{\hh_{n-1}(x_{t_{n-1}})^{\beta_{n-1}}}
\frac{M_n(x_{t_{n-1}},x_{t_n})}
{q_n(x_{t_n}\mid x_{t_{n-1}},y)}.
\]
By Assumption~\ref{ass:bounded-twist}, the first factor lies in $[\ell_n/u_{n-1},u_n/\ell_{n-1}]$ and is uniformly bounded on the finite grid.
It remains to bound the Gaussian density ratio.

Write
\[
M_n(\cdot\mid x_{t_{n-1}})=\cN(m_{M,n},a_n\I),
\qquad
q_n(\cdot\mid x_{t_{n-1}},y)=\cN(m_{q,n},c_n\I),
\]
where $a_n:=g(t_{n-1})^2\Delta_n>0$, $c_n:=a_n+\varepsilon$, and $\delta_n:=m_{q,n}-m_{M,n}$.
By Eqs.~\eqref{eq:prior-kernel} and \eqref{eq:guided-proposal},
\[
\begin{aligned}
\delta_n(x_{t_{n-1}})
&=w_g\,g(t_{n-1})^2\Delta_n\,\beta_{n-1}\nabla_x\log\hh_{n-1}(x_{t_{n-1}})\\
&=w_g\,g(t_{n-1})^2\Delta_n\,\nabla_x\log\!\big(\hh_{n-1}(x_{t_{n-1}})^{\beta_{n-1}}\big).
\end{aligned}
\]
If $\beta_{n-1}=0$ or $w_g=0$, this shift is zero; otherwise Assumption~\ref{ass:bounded-grad}, together with continuity of $g$ on $[0,1]$, the finite hyperparameter $w_g\geq 0$, and the finite grid, implies that $\|\delta_n(x_{t_{n-1}})\|$ is uniformly bounded over $(n,x_{t_{n-1}})$.
For $z:=x_{t_n}-m_{M,n}$,
\begin{align}
\log\frac{M_n(x_{t_{n-1}},x_{t_n})}{q_n(x_{t_n}\mid x_{t_{n-1}},y)}
&=
\frac{d}{2}\log\frac{c_n}{a_n}
-\frac{1}{2a_n}\|z\|^2
+\frac{1}{2c_n}\|z-\delta_n\|^2
\notag\\
&=
\frac{d}{2}\log\frac{c_n}{a_n}
-\frac{c_n-a_n}{2a_nc_n}\|z\|^2
-\frac{1}{c_n}\langle z,\delta_n\rangle
+\frac{1}{2c_n}\|\delta_n\|^2
\notag\\
&\le
\frac{d}{2}\log\frac{c_n}{a_n}
-\frac{\varepsilon}{2a_nc_n}\|z\|^2
+\frac{1}{c_n}\|z\|\,\|\delta_n\|
+\frac{1}{2c_n}\|\delta_n\|^2.
\label{eq:app-gaussian-ratio-bound-start}
\end{align}
For any $r\ge0$,
\[
-\frac{\varepsilon}{2a_nc_n}r^2
+\frac{\|\delta_n\|}{c_n}r
\le
\frac{a_n}{2c_n\varepsilon}\|\delta_n\|^2,
\]
which is the maximum of the scalar quadratic in $r$.
Applying this bound to $r=\|z\|$ in \eqref{eq:app-gaussian-ratio-bound-start} gives
\[
\log\frac{M_n(x_{t_{n-1}},x_{t_n})}{q_n(x_{t_n}\mid x_{t_{n-1}},y)}
\le
\frac{d}{2}\log\frac{c_n}{a_n}
+\frac{\|\delta_n\|^2}{2\varepsilon}.
\]
Because the time grid is finite, Assumption~\ref{ass:proposal-variance} makes the determinant term $\frac{d}{2}\log(c_n/a_n)$ finite uniformly over $n$.
The bound on $\delta_n$ derived from Assumption~\ref{ass:bounded-grad} makes the second term uniformly finite.
Therefore $M_n/q_n$ is uniformly bounded over $(n,x_{t_{n-1}},x_{t_n})$.
Combining this with Assumption~\ref{ass:bounded-twist}, each $\omega_n$ is positive and uniformly bounded.
The initial weight $\hh_0^{\beta_0}$ is also bounded by Assumption~\ref{ass:bounded-twist}.

\paragraph{Step 3: describe the exact distribution updates.}
We next write the exact updates that the particles approximate, so that we can isolate the error introduced at each step and follow its effect on the endpoint.
Write $H_n:=\hh_n^{\beta_n}$.
Starting from a state $x$ at step $n-1$, propagation and weighting give the following expected contribution of a bounded test function $f$,
\[
Q_nf(x):=\int q_n(x'\mid x,y)\omega_n(x,x')f(x')\,\dd x'
=\frac{1}{H_{n-1}(x)}\int M_n(x,\dd x')H_n(x')f(x').
\]
The equality follows by substituting the incremental weight, which cancels the proposal density.
For a probability law $\mu$ at step $n-1$, write $\mu(f):=\int f(x)\,\mu(\dd x)$.
Averaging these contributions and dividing by the total weight gives the exact updated expectation
\[
\Phi_n(\mu)(f):=\frac{\mu(Q_nf)}{\mu(Q_n\mathbf{1})}.
\]
Since $\beta_0=0$ and $H_0\equiv1$, the exact intermediate marginals start at $\eta_0=q_0$ and satisfy $\eta_n=\Phi_n(\eta_{n-1})$.
Step~1 identifies their terminal marginal as $\eta_N=\PosN$.

To describe the remaining updates from step $p$ to step $n$, let $M_{p+1:n}$ denote the composition of the prior kernels over those steps.
Repeated propagation and weighting compose the operators $Q_{p+1},\ldots,Q_n$.
As in Step~1, the intermediate twist factors cancel, giving
\[
Q_{p,n}f(x):=(Q_{p+1}\cdots Q_n f)(x)
=\frac{1}{H_p(x)}\int M_{p+1:n}(x,\dd z)H_n(z)f(z).
\]
Thus the exact continuation of a law $\mu$ from step $p$ to step $n$ is
\[
\Phi_{p,n}(\mu)(f):=\frac{\mu(Q_{p,n}f)}{\mu(Q_{p,n}\mathbf{1})},
\]
with $Q_{n,n}=I$ and $\Phi_{n,n}(\mu)=\mu$ when no updates remain.
In particular, $\Phi_{n-1,n}=\Phi_n$.

At the endpoint, $H_N=\Ly$ and $h_p(x)=\int M_{p+1:N}(x,\dd z)\Ly(z)$.
Normalizing this likelihood-weighted prior transition gives the posterior transition kernel
\[
P^y_{p,N}(x,\dd z)
=\frac{M_{p+1:N}(x,\dd z)\Ly(z)}{h_p(x)}.
\]
Consequently, the remaining weighted contribution factors as
\[
Q_{p,N}f(x)
=\frac{1}{H_p(x)}\int M_{p+1:N}(x,\dd z)\Ly(z)f(z)
=\frac{h_p(x)}{H_p(x)}P^y_{p,N}f(x),
\]
where $P^y_{p,N}f(x)$ is the conditional expectation of the endpoint function $f$ under this kernel.
Define the total remaining weight by $u_{p,N}:=Q_{p,N}\mathbf{1}=h_p/H_p$.
The exact continuation therefore first reweights the current law by $u_{p,N}$ and then applies the posterior transition,
\[
\Phi_{p,N}(\mu)(f)
=\frac{\mu(u_{p,N}P^y_{p,N}f)}{\mu(u_{p,N})}.
\]
Here $u_{N,N}\equiv1$ and $P^y_{N,N}=I$.
Assumption~\ref{ass:bounded-twist} gives $\ell_N/u_p\le u_{p,N}(x)\le u_N/\ell_p$, so its relative variation is $\sup u_{p,N}/\inf u_{p,N}=R_{p,N}<\infty$.
This factorization separates the weight correction and the posterior transition whose effects on error we bound in Step~5.

\paragraph{Step 4: bound the error introduced at one step.}
We now compare one particle update with the exact update from the same starting particles.
Write the weighted particle estimate of a test function as $\eta_n^K(f):=\sum_{i=1}^KW_n^{(i)}f(X_{t_n}^{(i)})$.
For $n\ge1$, define its comparison law by
\[
\tilde\eta_n^K:=\Phi_n(\eta_{n-1}^K),
\qquad
\tilde\eta_0^K:=\eta_0.
\]
The difference $\eta_n^K-\tilde\eta_n^K$ isolates the sampling error added at step $n$ while holding the previous particle approximation fixed.
Let $\cF_{n-1}^K$ be the sigma-field generated by the particle history through step $n-1$; this history determines the ESS resampling decision.
For every bounded $\cF_{n-1}^K$-measurable test function $f$, we will show
\begin{equation}\label{eq:app-local-error}
\EE\Big[\big|(\eta_n^K-\tilde\eta_n^K)(f)\big|^2\,\Big|\,\cF_{n-1}^K\Big]^{1/2}
\le\frac{\alpha_n}{\sqrt K}\,\osc(f),
\qquad n\in\{0,\ldots,N\},
\end{equation}
where $\alpha_n$ is deterministic and independent of $K$.
At $n=0$, the bound is unconditional with deterministic $f$.
The independent samples from $q_0$ have equal weights, so $\Var[\eta_0^K(f)]\le\osc(f)^2/(4K)$ and we can take $\alpha_0=1/2$.

For $n\ge1$, the first two conditional moments of an incremental weight are
\[
m_n(x):=\int q_n(x'\mid x,y)\omega_n(x,x')\,\dd x'=Q_n\mathbf{1}(x),
\qquad
v_n(x):=\int q_n(x'\mid x,y)\omega_n(x,x')^2\,\dd x'.
\]
Assumption~\ref{ass:bounded-twist} bounds the mean weight, and Step~2 bounds each realized weight by a finite constant $\bar\omega_n$.
These bounds imply
\[
\underline m_n:=\frac{\ell_n}{u_{n-1}}\le m_n(x)\le\frac{u_n}{\ell_{n-1}}=:\overline m_n,
\qquad
v_n(x)\le\bar\omega_nm_n(x)\le\kappa_nm_n(x)^2,
\]
where $\kappa_n:=\bar\omega_n/\underline m_n$ and $b_n:=\overline m_n/\underline m_n$.
We apply Lemma~\ref{lem:ratio} with $F_i:=f(X_{t_n}^{(i)})$, choosing $Z_i$ according to the resampling decision.

If resampling occurs, the reset weights are equal, so we can take $Z_i:=\omega_n(\bar X_{t_{n-1}}^{(i)},X_{t_n}^{(i)})$ and obtain $\hat\theta=\eta_n^K(f)$.
Stratified resampling uses independent uniforms on the $K$ strata, followed by independent proposal draws, making the pairs $(Z_i,F_i)$ independent conditional on $\cF_{n-1}^K$.
Its expected ancestor counts satisfy
\[
\sum_{i=1}^K\PP(A_{n-1}^{(i)}=j\mid\cF_{n-1}^K)=KW_{n-1}^{(j)}.
\]
Summing over the ancestors and proposal draws therefore gives
\begin{align*}
\theta
&=\frac{\sum_i\EE[Z_iF_i\mid\cF_{n-1}^K]}{\sum_i\EE[Z_i\mid\cF_{n-1}^K]}
=\frac{K\eta_{n-1}^K(Q_nf)}{K\eta_{n-1}^K(m_n)}
=\tilde\eta_n^K(f),\\
\frac{\sum_i\EE[Z_i^2\mid\cF_{n-1}^K]}{(\sum_i\EE[Z_i\mid\cF_{n-1}^K])^2}
&=\frac{\eta_{n-1}^K(v_n)}{K\eta_{n-1}^K(m_n)^2}
\le\frac{\kappa_nb_n^2}{K}.
\end{align*}

If resampling is skipped, we include the previous particle weights by setting
\[
Z_i:=W_{n-1}^{(i)}\omega_n(X_{t_{n-1}}^{(i)},X_{t_n}^{(i)}).
\]
The proposal draws again give conditionally independent pairs and $\hat\theta=\eta_n^K(f)$.
In this case,
\begin{align*}
\theta
&=\frac{\eta_{n-1}^K(Q_nf)}{\eta_{n-1}^K(m_n)}
=\tilde\eta_n^K(f),\\
\frac{\sum_i\EE[Z_i^2\mid\cF_{n-1}^K]}{(\sum_i\EE[Z_i\mid\cF_{n-1}^K])^2}
&=\frac{\sum_i(W_{n-1}^{(i)})^2v_n(X_{t_{n-1}}^{(i)})}{\eta_{n-1}^K(m_n)^2}
\le\kappa_nb_n^2\sum_i(W_{n-1}^{(i)})^2
\le\frac{\kappa_nb_n^2}{\rho K}.
\end{align*}
The last inequality follows because skipping resampling requires $\mathrm{ESS}_{n-1}\ge\rho K$.
Since $\rho\le1$, both cases satisfy the second-moment bound $\kappa_nb_n^2/(\rho K)$.
Lemma~\ref{lem:ratio}, applied conditionally on the particle history, proves Eq.~\eqref{eq:app-local-error} with $\alpha_n=b_n\sqrt{8\kappa_n/\rho}$.

\paragraph{Step 5: propagate the one-step errors to the endpoint.}
We finally express the terminal error as a sum of the one-step sampling errors, each followed by the remaining exact updates.
For $p\ge1$, the definition of $\tilde\eta_p^K$ gives
\[
\Phi_{p,N}(\tilde\eta_p^K)
=\Phi_{p,N}\big(\Phi_p(\eta_{p-1}^K)\big)
=\Phi_{p-1,N}(\eta_{p-1}^K).
\]
The second equality combines the exact update from step $p-1$ to step $p$ with the remaining exact updates.
When we sum the differences between the particle laws and their comparison laws, these consecutive terms cancel.
Using $\Phi_{N,N}(\eta_N^K)=\eta_N^K$ and $\Phi_{0,N}(\eta_0)=\eta_N$, we obtain
\[
\eta_N^K-\eta_N
=\sum_{p=0}^N\Big[\Phi_{p,N}(\eta_p^K)-\Phi_{p,N}(\tilde\eta_p^K)\Big].
\]

To apply Step~4 to each summand, we rewrite the propagated difference as a one-step error evaluated on a transformed test function.
Fix $p$ and write $\mu:=\eta_p^K$, $\nu:=\tilde\eta_p^K$, and $B:=\Phi_{p,N}(\nu)(\varphi)$.
The factorization in Step~3 gives
\begin{align*}
\Phi_{p,N}(\mu)(\varphi)-\Phi_{p,N}(\nu)(\varphi)
&=\frac{\mu(u_{p,N}P^y_{p,N}\varphi)}{\mu(u_{p,N})}-B\\
&=\frac{(\mu-\nu)\big(u_{p,N}(P^y_{p,N}\varphi-B)\big)}{\mu(u_{p,N})},
\end{align*}
where the last equality uses $\nu\big(u_{p,N}(P^y_{p,N}\varphi-B)\big)=0$.
Define
\[
D\varphi(x):=\frac{u_{p,N}(x)}{\nu(u_{p,N})}\big(P^y_{p,N}\varphi(x)-B\big).
\]
Then the propagated difference becomes
\[
\Phi_{p,N}(\mu)(\varphi)-\Phi_{p,N}(\nu)(\varphi)
=\frac{\nu(u_{p,N})}{\mu(u_{p,N})}(\mu-\nu)(D\varphi).
\]
Because $\nu=\tilde\eta_p^K$ depends only on the particle history through step $p-1$, the function $D\varphi$ is $\cF_{p-1}^K$-measurable and can be used in Eq.~\eqref{eq:app-local-error}.
At $p=0$, it is deterministic.

It remains to bound the prefactor and the oscillation of this test function.
The definition of $R_{p,N}$ gives
\[
\frac{\nu(u_{p,N})}{\mu(u_{p,N})}
\le\frac{\sup u_{p,N}}{\inf u_{p,N}}=R_{p,N}.
\]
Also, $B$ is a weighted average of $P^y_{p,N}\varphi$, so it lies between that function's infimum and supremum.
Since $u_{p,N}(x)/\nu(u_{p,N})\le R_{p,N}$, we have
\[
\osc(D\varphi)
\le2\sup_x|D\varphi(x)|
\le2R_{p,N}\osc(P^y_{p,N}\varphi)
\le2R_{p,N}\varrho(P^y_{p,N})\osc(\varphi).
\]
Applying the conditional bound in Eq.~\eqref{eq:app-local-error} and then taking expectations yields
\[
\Big\|\Phi_{p,N}(\eta_p^K)(\varphi)-\Phi_{p,N}(\tilde\eta_p^K)(\varphi)\Big\|_{L_2}
\le\frac{2\alpha_pR_{p,N}^2\varrho(P^y_{p,N})}{\sqrt K}\,\osc(\varphi).
\]
Finally, the triangle inequality in $L_2$ sums these propagated errors,
\[
\|\eta_N^K(\varphi)-\eta_N(\varphi)\|_{L_2}
\le\frac{2\,\osc(\varphi)}{\sqrt K}\sum_{p=0}^N\alpha_pR_{p,N}^2\varrho(P^y_{p,N}).
\]
Since $\eta_N=\PosN$ by Step~1, this is Eq.~\eqref{eq:finite-k} with the stated constant $C_N$.
\end{proof}

\subsection{Proof of Remark~\ref{rem:beta-star}}\label{app:proof-beta-star}

Under the assumptions of Remark~\ref{rem:beta-star}, $X_1| X_t=x\sim\cN(\mustar{1| t}(x),\sigma_{1| t}^2\I)$, so the exact twist is the Gaussian convolution
\[
h_t(x)=\int\cN(y;x_1,\sigma_y^2\I)\,\cN(x_1;\mustar{1\mid t}(x),\sigma_{1\mid t}^2\I)\,\dd x_1
=\cN\big(y;\mustar{1\mid t}(x),(\sigma_y^2+\sigma_{1\mid t}^2)\I\big).
\]
The plug-in twist is $\hh_t(x)=\cN(y;\mustar{1| t}(x),\sigma_y^2\I)$, so for any $\beta>0$,
\[
\hh_t(x)^{\beta}=(2\pi\sigma_y^2)^{-\beta d/2}\exp\Big(-\frac{\beta\,\|y-\mustar{1\mid t}(x)\|^2}{2\sigma_y^2}\Big).
\]
With $\beta=\beta^\star(t)=\sigma_y^2/(\sigma_y^2+\sigma_{1| t}^2)$, the exponent equals $-\|y-\mustar{1| t}(x)\|^2/\{2(\sigma_y^2+\sigma_{1| t}^2)\}$; hence $\hh_t(x)^{\beta^\star(t)}=c_t\,h_t(x)$ with $c_t=(2\pi\sigma_y^2)^{-\beta^\star(t)d/2}\{2\pi(\sigma_y^2+\sigma_{1| t}^2)\}^{d/2}$ independent of $x$.
Taking logarithms and gradients gives $\nabla\log\hh_t^{\beta^\star(t)}=\nabla\log h_t$, so with $w_g=1$ the guidance term in Eq.~\eqref{eq:guided-proposal} equals the exact $h$-transform drift correction $g(t)^2\nabla\log h_t$.
The factors $c_{t_n}/c_{t_{n-1}}$ multiply every incremental weight \eqref{eq:incremental-weight} by the same particle-independent constant and cancel after normalization, so the tempered plug-in twist and the exact twist induce the same sampler.

The result also extends to observing only a subset of the coordinates of $x_1$.
In the model $y=\A x_1+\eta$ with $\eta\sim\cN(0,\sigma_y^2\I)$, the matrix $\A$ selects the observed coordinates, such as the unmasked pixels of an image.
Since $\A\A^\top=\I$ in the observation space, the same computation gives $h_t(x)=\cN\big(y;\A\mustar{1| t}(x),(\sigma_y^2+\sigma_{1| t}^2)\I\big)$, so $\hh_t(x)^{\beta^\star(t)}\propto h_t(x)$ holds with the same $\beta^\star(t)$.

\section{Pretraining details}\label{app:pretraining}

This appendix provides the SB pretraining convention assumed by the main text.
We follow the simulation-free score-and-flow matching (SF2M) framework~\citep{tong2024simulation} but with a different parameterization, detailed below.

\paragraph{Score and flow matching.}
Under a Brownian reference with diffusion coefficient $\sigma$, the Schr\"odinger bridge's endpoint coupling is the entropic optimal transport (OT) coupling with squared Euclidean cost and regularization strength $2\sigma^2$.
For endpoints $(x_0,x_1)$ drawn from this coupling and $t\sim\mathrm{Unif}(0,1)$, the corresponding Brownian bridge sample is
\[
x_t=(1-t)x_0+tx_1+\sigma\sqrt{t(1-t)}\,\eps,
\qquad \eps\sim\cN(0,\I).
\]
Writing $m_t=(1-t)x_0+tx_1$, its conditional probability-flow velocity and score are
\[
u^\circ_t(x\mid x_0,x_1)
=x_1-x_0+\frac{1-2t}{2t(1-t)}(x-m_t),
\qquad
s^\circ_t(x\mid x_0,x_1)
=-\frac{x-m_t}{\sigma^2t(1-t)}.
\]
The standard SF2M objective jointly regresses these two targets~\citep{tong2024simulation},
\[
\EE\big[
\|\vthe(t,X_t)-u^\circ_t(X_t\mid X_0,X_1)\|^2
+w(t)^2\|\sthe(t,X_t)-s^\circ_t(X_t\mid X_0,X_1)\|^2
\big],
\]
where $w(t)>0$ weights the score loss and the expectation is over the pair $(X_0,X_1)$, time, and bridge noise.
In our implementation, we approximate the endpoint coupling using minibatch exact OT, following \citet{tong2024simulation}.

\paragraph{Displacement and noise parameterization.}
Holding $x_0$, $x_1$, and $\eps$ fixed, differentiating the interpolation in $t$ gives the conditional probability-flow velocity
\[
u^\circ_t(x_t\mid x_0,x_1)
=\underbrace{x_1-x_0}_{\text{displacement }D}
+\underbrace{\sigma\,\frac{1-2t}{2\sqrt{t(1-t)}}\,\eps}_{\text{noise correction term}},
\]
whose noise correction coefficient $\sigma(1-2t)/(2\sqrt{t(1-t)})$ diverges as $t\to 0$ or $t\to 1$.
Directly regressing this velocity would inherit this divergence, so we predict $D$ and $\eps$ separately and reconstruct the velocity $\vthe(t,x)$ analytically.
We instantiate a displacement channel $d_\theta(t,x)$ targeting $D=x_1-x_0$ and a noise channel $\eps_\theta(t,x)$ targeting $\eps$.
We consider a \emph{joint} model with a shared UNet backbone and two heads, and a \emph{dual} model that uses two independent UNets.
Both are trained with the same loss, and experiments use the joint version for efficiency unless noted.
The loss is
\[
\EE\left[
\|d_\theta(t,x_t)-(x_1-x_0)\|^2
+\lambda(t)\|\eps_\theta(t,x_t)-\eps\|^2
\right],
\]
where $\lambda(t)>0$ weights the noise regression and the expectation is over pairs $(x_0,x_1)$, times, and Gaussian noise.

\paragraph{Recovering velocity and score.}
The noise head is converted to the score convention used in the main text, $\sthe(t,x)\approx\nabla\log p_t(x)$.
Since $\nabla_{x_t}\log p(x_t| x_0,x_1)=-\eps/(\sigma\sqrt{t(1-t)})$, the noise head converts to the score as
\[
\sthe(t,x)
=-\frac{\eps_\theta(t,x)}{\sigma\sqrt{t(1-t)}}.
\]
The velocity is reconstructed from the displacement identity
\[
\EE[D\mid X_t=x]
=
\vstar_t(x)+\frac{1-2t}{2}\sigma^2\sstar_t(x).
\]
Thus the implemented velocity channel is
\[
\vthe(t,x)
=
d_\theta(t,x)-\frac{1-2t}{2}\sigma^2\sthe(t,x).
\]
With this convention, the endpoint predictor used by \ours{} is exactly Eq.~\eqref{eq:hmu1t}.
If an implementation stores the scaled score $s_{\mathrm{sc}}(t,x)=(\sigma^2/2)\sthe(t,x)$, then Eq.~\eqref{eq:hmu1t} is evaluated as
\[
\hmu{1\mid t}(x)
=
x+(1-t)\vthe(t,x)+(1-t)s_{\mathrm{sc}}(t,x).
\]

\section{Implementation details}\label{app:implementation}

\subsection{Two-dimensional toy benchmark}\label{app:toy-benchmark}

For the toy setup in Section~\ref{sec:exp:toy}, we use the independent coupling $X_0\sim q_0,\,X_1\sim q_1$ and the Brownian-bridge interpolant $X_t=(1-t)X_0+tX_1+\sigma\sqrt{t(1-t)}\,\epsilon$.
Under this interpolant the component marginal is $p_t(x| k)=\cN(x;t\mu_k,\sigma_t^2\I)$ with $\sigma_t^2=(1-t)^2+\sigma^2t(1-t)+t^2/4$, and the responsibilities are
\[
\gamma_k(t,x)=\frac{\pi_k\,\cN(x;t\mu_k,\sigma_t^2\I)}{\sum_j\pi_j\,\cN(x;t\mu_j,\sigma_t^2\I)}.
\]
The analytic score and probability-flow velocity used by \ours{}, DPS-SB, and IS are
\[
\sstar_t(x)=\sum_k\gamma_k(t,x)\,\frac{-(x-t\mu_k)}{\sigma_t^2},\qquad
\vstar_t(x)=\sum_k\gamma_k(t,x)\!\left[\mu_k+\frac{\dot{\sigma}_t^2}{2\sigma_t^2}(x-t\mu_k)\right],
\]
with $\dot{\sigma}_t^2=-2(1-t)+\sigma^2(1-2t)+t/2$.

FMPS and TFTF instead use the deterministic flow-matching path between the same distributions $q_0$ and $q_1$, whose component variance is $\tilde{\sigma}_t^2=(1-t)^2+t^2/4$ (with $\tilde{\gamma}_k$ defined by substituting $\tilde{\sigma}_t^2$ above) and whose velocity is
\[
v^{\rm FM}_t(x)=\sum_k\tilde{\gamma}_k(t,x)\!\left[\mu_k+\frac{t/4-(1-t)}{\tilde{\sigma}_t^2}(x-t\mu_k)\right].
\]

For TDS, we use the standard variance-preserving (VP) diffusion marginal with the same endpoint mixture $q_1$:
\[
X_t=a_tX_1+r_tZ,\qquad a_t=\sin(\pi t/2),\quad r_t=\cos(\pi t/2),\quad Z\sim\cN(0,\I).
\]
Thus, $p_t(x| k)=\cN(x;a_t\mu_k,V_t\I)$ with $V_t=a_t^2/4+r_t^2$, and
\[
\gamma_k^{\rm VP}(t,x)=
\frac{\pi_k\,\cN(x;a_t\mu_k,V_t\I)}
{\sum_j\pi_j\,\cN(x;a_t\mu_j,V_t\I)}.
\]
The analytic VP score used by TDS is
\[
s^{\rm VP}_t(x)
=\sum_k\gamma_k^{\rm VP}(t,x)\frac{-(x-a_t\mu_k)}{V_t}.
\]
where $\dot a_t=(\pi/2)r_t$, $\dot r_t=-(\pi/2)a_t$, and $\dot V_t=2(a_t\dot a_t/4+r_t\dot r_t)$.
The Tweedie endpoint predictor in TDS is
\[
\mu_{1\mid t}^{\rm VP}(x)
=\sum_k\gamma_k^{\rm VP}(t,x)\left[\mu_k+\frac{a_t/4}{V_t}(x-a_t\mu_k)\right].
\]

\subsection{Unconditional models and likelihoods}\label{app:pretrained-models}

The unconditional SB models used by \ours{}, DPS-SB, and IS are pretrained as described in Appendix~\ref{app:pretraining}, with target $q_1$ given by MNIST, CIFAR-10, or CelebA-HQ, source $q_0=\cN(0,\I)$ in pixel space, and Brownian reference $\sigma=1$ for MNIST and CIFAR-10 and $\sigma=0.25$ for CelebA-HQ.
For a fair comparison, FMPS~\citep{song2024fmps} and TFTF~\citep{qu2026tftf} use the same unconditional flow matching checkpoints as TFTF, rather than the velocity field learned by the SB.
We initialize the displacement head $d_\theta(t,x)$ in our SB pretraining from these checkpoints, so the compared methods have approximately matched network capacity.

\paragraph{Base models.}
For MNIST, the initialization is the pretrained U-Net used by TFTF~\citep{qu2026tftf}.
For CIFAR-10 and CelebA-HQ, the initialization is the pretrained rectified flow model of \citet{liu2023flow} from its public RectifiedFlow release.\footnote{\url{https://github.com/gnobitab/RectifiedFlow}}

Table~\ref{tab:parity} reports the parameter counts and the unconditional FID (10k samples against the training set) of the priors used by the compared methods.
The SB prior adds one score head to the rectified-flow architecture, so the parameter counts match within $0.1\%$, and its unconditional FID is slightly worse than that of the flow and diffusion checkpoints; the conditional gains of \ours{} therefore do not come from a larger or better base model.
The SB prior is initialized from the TFTF checkpoint and fine-tuned for $40{,}000$ steps with the additional score head, which takes about $5.5$ GPU hours on an RTX 5090 for CIFAR-10, a one-time overhead relative to using a pretrained flow or diffusion model directly.

\begin{table}[!htbp]
  \caption{Parameter count and unconditional FID (10k samples) of the pretrained priors.
  NA marks the absence of a released checkpoint.}
  \label{tab:parity}
  \centering
  \setlength{\tabcolsep}{4pt}
  \begin{adjustbox}{max width=\textwidth}
  \begin{tabular}{@{}lccc@{}}
    \toprule
    Prior model & MNIST params / FID & CIFAR-10 params / FID & CelebA-HQ params / FID \\
    \midrule
    Rectified flow (TFTF checkpoint) & 14.9777M / 5.22 & 61.8044M / 7.26 & 65.5745M / 48.40 \\
    DDPM / iDDPM (TDS checkpoint) & 14.4154M / 2.27 & 52.5464M / 7.33 & NA \\
    SB (ours) & 14.9782M / 6.53 & 61.8079M / 8.49 & 65.6160M / 51.54 \\
    \bottomrule
  \end{tabular}
  \end{adjustbox}
\end{table}

\paragraph{Likelihood models.}
For MNIST and CIFAR-10 class-conditional sampling, the likelihood classifier is a frozen ResNet-50~\citep{he2016deep} from the minimal-diffusion release.\footnote{\url{https://github.com/VSehwag/minimal-diffusion}}
For inpainting, super-resolution, and deblurring, we use the Gaussian likelihood $\Ly(x_1)=\cN(y;\A x_1,\sigma_y^2\I)$.
The observation operator $\A$ selects the observed pixels for inpainting, performs $4\times$ bicubic downsampling for super-resolution, and applies a $9\times9$ Gaussian blur with standard deviation $2$ for deblurring.
The observation noise standard deviation is $\sigma_y=0.05$ for MNIST inpainting, $0.1$ for CIFAR-10 inpainting, and $0.03$ for both CelebA-HQ restoration tasks.
For CelebA-HQ text-prompted image generation, the likelihood is computed with a ConvNext-XXLarge CLIP model from OpenCLIP~\citep{radford2021clip, liu2022convnet, cherti2023reproducible}.\footnote{\url{https://github.com/mlfoundations/open_clip}}
The likelihood takes the form $\Ly(x_1)\propto\exp\{\lambda\,\operatorname{sim}(\phi_I(x_1),\phi_T(y))\}$, where $\phi_I$ and $\phi_T$ are image and text embeddings, $\operatorname{sim}$ is cosine similarity, and $\lambda$ sets the energy scale.

\paragraph{External evaluators.}
For MNIST, external accuracy is measured by a separately trained Super Learner classifier.\footnote{\url{https://github.com/Curt-Park/handwritten_digit_recognition}}
For CIFAR-10, external accuracy is measured by a DenseNet-121 classifier~\citep{huang2017densely} from a public CIFAR-10 model release.\footnote{\url{https://github.com/huyvnphan/PyTorch_CIFAR10}}
For CelebA-HQ, external prompt alignment is measured by the attribute classifiers released with TransEditor~\citep{xu2022transeditor}.\footnote{\url{https://github.com/BillyXYB/TransEditor}}

\subsection{Hyperparameters and metrics}\label{app:hyperparameters}\label{app:additional-baseline-settings}

\paragraph{Toy example.}
\ours{} uses guidance strength $w_g=0.5$, bridge-shaped diffusion schedule $g(t)=c_g\sigma\sqrt{t(1-t)+\delta_g}$ with $c_g=1$, ESS threshold $\rho=0.8$, and systematic resampling activated from $t=0.1$ onward, with the small constants $\delta_g=\varepsilon=10^{-6}$. Guidance-baseline strengths are tuned by grid search: $r=0.2$ for FMPS and $\zeta=0.5$ for DPS-SB. TFTF uses the default three-phase schedule~\citep{qu2026tftf}.

\paragraph{Image experiments.}
\ours{} uses guidance strength $w_g=1$ for class-conditional sampling and text-prompted image generation and $w_g=2$ for inpainting, the same bridge-shaped schedule with $c_g=4$, and ESS threshold $\rho=0.5$; the remaining settings are same to the toy example. Guidance-baseline strengths are still tuned by grid search: $r=1$ for FMPS and $\zeta=200$ for DPS-SB on MNIST and CIFAR-10; $r=0.003$ for FMPS and $\zeta=50$ for DPS-SB on CelebA-HQ. 

\paragraph{CIFAR-10 inpainting baselines.}
DAPS and SITCOM use the unconditional iDDPM prior used by TDS; PnP-Flow uses the rectified-flow prior used by TFTF.
DAPS uses 200 annealing steps, five ODE substeps and 100 Langevin updates per annealing step, a Langevin learning rate of $5\times10^{-5}$, and $\tau=0.01$.
SITCOM uses 20 DDIM steps with up to 30 Adam iterations per step, a learning rate of $0.02$, a projection scale of $0.5$, and a noise-level parameter of $0.1$.
PnP-Flow uses 100 steps and a learning rate of $1.0$, with $\alpha=0.5$ for half-image masks and $\alpha=0.01$ for random $50\%$ masks.

\paragraph{CelebA-HQ restoration baselines.}
PnP-Flow, Flower, FlowDPS, TFTF, and FMPS share the rectified-flow prior in Table~\ref{tab:parity}, their native model class.
DAPS and SITCOM are excluded from this comparison because no compatible released CelebA-HQ diffusion checkpoint is available.
PnP-Flow uses 100 steps and a learning rate of $1.0$, with $\alpha=0.01$ for deblurring and $\alpha=0.3$ for super-resolution.
Flower uses 100 steps with up to 50 conjugate-gradient iterations per step.
The stochastic time progression is $x\leftarrow(t+\delta)x^\star+(1-t-\delta)z_0$, with $\delta=1/100$ and time grid $t=\delta i$ for $i=0,\ldots,99$.
We retain the original time parameterization and the $[-1,1]$ image convention without additional output rescaling.
FlowDPS uses 28 solver iterations, each with three data-consistency gradient updates of step size 15.
The baselines use their released configurations, with no additional tuning of learning rates or other hyperparameters on these two restoration tasks.

\paragraph{Metrics.}
All metrics are computed on endpoint samples after resampling under the final particle weights.
For class-conditional sampling, each metric is averaged over $K$ particles, $10$ classes, and $100$ replications per-class.
Density and Coverage \citep{naeem2020reliable} are computed per class between the generated samples and $1{,}000$ real training images of that class in Inception feature space with $k=5$ nearest neighbours, and averaged over classes. 
For text-prompted image generation, each metric is averaged over $K$ particles, $40$ prompts, and $5$ replications per-prompt.
LPIPS diversity is the mean pairwise LPIPS distance \citep{zhang2018unreasonable} among the samples generated for the same prompt.
For inpainting, we use half-image and random $50\%$ masks and evaluate SSIM on the full reconstruction.
For the CelebA-HQ inverse problems, PSNR and SSIM are averaged over $100$ test images and three independent runs per image, with standard deviations computed over all images and runs.

The main image experiments were conducted on an NVIDIA RTX 5090 GPU.

\section{Additional results}\label{app:additional}

\subsection{Computation cost}\label{app:cost}

Table~\ref{tab:cost} reports the number of forward and backward backbone evaluations (NFEs), the wall-clock time, and the peak GPU memory of one CIFAR-10 class-conditional run at $K=64$ and $N=800$, using the configurations of Table~\ref{tab:class-cond-k64}.
\ours{} obtains $(\vthe,\sthe)$ from one shared-backbone evaluation and performs one input-gradient backward pass per particle and time step; its wall-clock time and peak memory are comparable to those of the gradient-based baselines, and its $\cO(KN)$ cost is of the same order as that of TFTF, TDS, and FMPS.

\begin{table}[!htbp]
  \caption{Computation cost of one CIFAR-10 class-conditional run at $K=64$ and $N=800$: forward and backward backbone evaluations, wall-clock time, and peak GPU memory.}
  \label{tab:cost}
  \centering
  
  \setlength{\tabcolsep}{6pt}
  \begin{adjustbox}{max width=\textwidth}
  \begin{tabular}{@{}lcccc@{}}
    \toprule
    Method & forward NFE & backward NFE & time (s) & memory (GB) \\
    \midrule
    FMPS & 51{,}200 & 51{,}200 & 48.1 & 6.9 \\
    TFTF & 51{,}200 & 22{,}336 & 46.1 & 6.9 \\
    DPS-SB & 51{,}200 & 51{,}200 & 50.9 & 6.7 \\
    IS & 51{,}200 & 0 & 20.7 & 1.3 \\
    TDS & 51{,}200 & 51{,}200 & 45.3 & 7.0 \\
    \ours{} & 51{,}264 & 51{,}264 & 50.1 & 6.7 \\
    \bottomrule
  \end{tabular}
  \end{adjustbox}
\end{table}

\subsection{Joint particle and time-step allocation}\label{app:fixed-nfe-allocation}

We vary the particle count $K$ and time-step count $N$ jointly for \ours{} under total forward-plus-backward backbone evaluation (NFE) budgets $B\in\{25632, 102528\}$.
For each $N\in\{100,200,400,800\}$, we choose $K=\lfloor B/[2(N+1)]\rfloor$, the largest integer particle count with $\mathrm{NFE}=2(N+1)K\le B$; the actual counts are reported in Table~\ref{tab:fixed-nfe-allocation}.

Figure~\ref{fig:fixed-nfe-grid} and Table~\ref{tab:fixed-nfe-allocation} report the quality metrics and computational costs, respectively.
We evaluate the $K$ endpoint samples after systematic resampling from the terminal weights.
The results show that the allocation between $K$ and $N$ affects quality under a fixed computational budget, and a relatively large time-step count $N$ (e.g., $N=400$) can be beneficial.

\begin{figure}[!htbp]
  \centering
  \includegraphics[width=\linewidth]{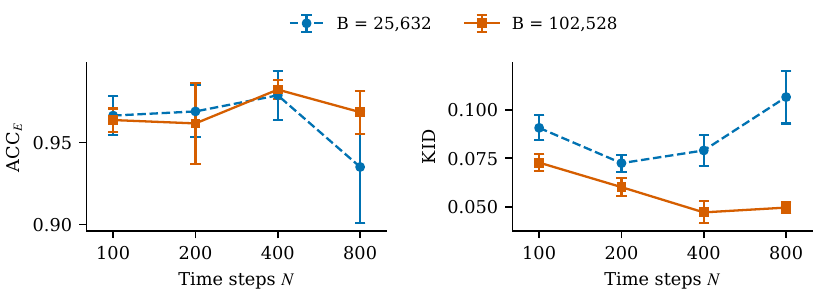}
  \caption{CIFAR-10 quality under two total backbone-NFE ceilings $B$.
Points and error bars show the mean and sample standard deviation based on ten classes and five repetitions per-class.
Higher ACC$_E$ and lower KID are preferred.}
  \label{fig:fixed-nfe-grid}
\end{figure}

\begin{table}[!htbp]
  \centering
  \caption{Particle allocation and sampling costs.
Memory gives mean peak allocation.}
  \label{tab:fixed-nfe-allocation}
  \setlength{\tabcolsep}{6pt}
  \begin{adjustbox}{max width=\textwidth}
  \begin{tabular}{@{}rrrcc@{}}
    \toprule
    $N$ & $K$ & Actual NFE & Time (s) & Memory (GiB) \\
    \midrule
    \multicolumn{5}{c}{$B=25{,}632$} \\
    \midrule
    $100$ & $126$ & $25{,}452$ & $45.81_{\pm0.07}$ & $6.708$ \\
    $200$ & $63$ & $25{,}326$ & $44.49_{\pm0.19}$ & $6.678$ \\
    $400$ & $31$ & $24{,}862$ & $42.16_{\pm0.07}$ & $3.468$ \\
    $800$ & $16$ & $25{,}632$ & $80.11_{\pm0.02}$ & $2.175$ \\
    \midrule
    \multicolumn{5}{c}{$B=102{,}528$} \\
    \midrule
    $100$ & $507$ & $102{,}414$ & $177.80_{\pm0.20}$ & $6.782$ \\
    $200$ & $255$ & $102{,}510$ & $189.03_{\pm20.72}$ & $6.732$ \\
    $400$ & $127$ & $101{,}854$ & $183.21_{\pm0.18}$ & $6.708$ \\
    $800$ & $64$ & $102{,}528$ & $178.78_{\pm0.28}$ & $6.701$ \\
    \bottomrule
  \end{tabular}
  \end{adjustbox}
\end{table}

\subsection{Qualitative samples}\label{app:inpainting-mean}

Figure~\ref{fig:inpainting-mnist-compare} compares MNIST half-mask inpainting, and Figure~\ref{fig:celebahq-recon} shows \ours{} reconstructions for the super-resolution and deblurring tasks of Table~\ref{tab:solvers}.

\begin{figure}[!t]
  \centering
  \includegraphics[width=0.55\linewidth]{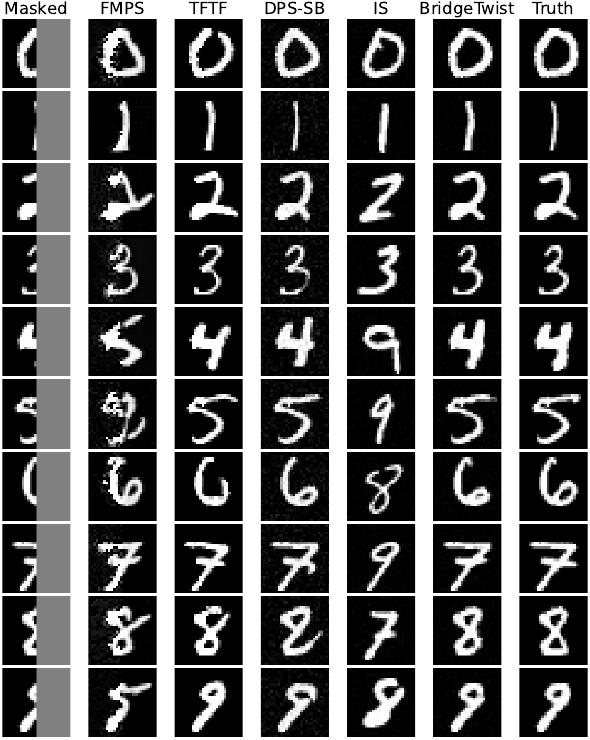}
  \caption{MNIST half-mask inpainting comparison. Each row shows one digit example; columns show the masked input, FMPS, TFTF, DPS-SB, IS, \ours{}, and the ground truth.}
  \label{fig:inpainting-mnist-compare}
\end{figure}

\begin{figure}[!t]
  \centering
  \setlength{\tabcolsep}{1.5pt}
  \renewcommand{\arraystretch}{0.5}
  \begin{tabular}{@{}c@{\hspace{4pt}}cccc@{\hspace{6pt}}ccc@{}}
    & {\scriptsize Input} & {\scriptsize \ours{}} & {\scriptsize Ground truth} & & {\scriptsize Input} & {\scriptsize \ours{}} & {\scriptsize Ground truth} \\
    \raisebox{0.06\textwidth}{\rotatebox{90}{\scriptsize $4\times$ SR}} &
    \includegraphics[width=0.145\textwidth]{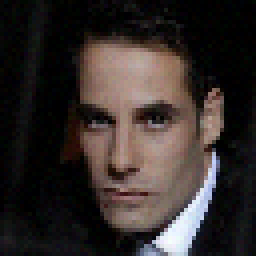} &
    \includegraphics[width=0.145\textwidth]{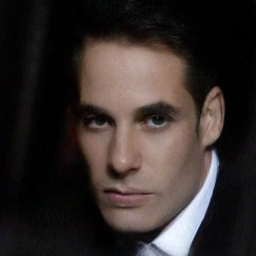} &
    \includegraphics[width=0.145\textwidth]{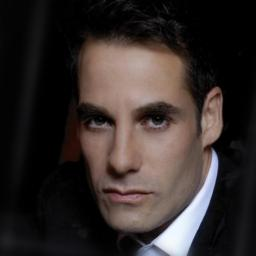} & &
    \includegraphics[width=0.145\textwidth]{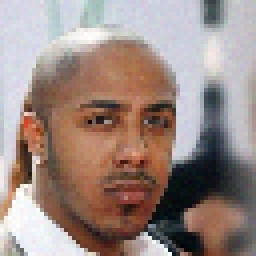} &
    \includegraphics[width=0.145\textwidth]{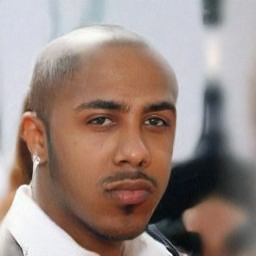} &
    \includegraphics[width=0.145\textwidth]{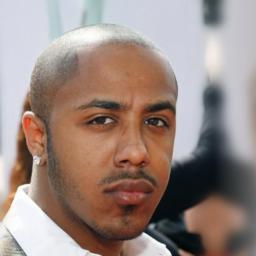} \\
    \raisebox{0.05\textwidth}{\rotatebox{90}{\scriptsize Deblurring}} &
    \includegraphics[width=0.145\textwidth]{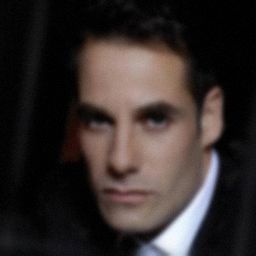} &
    \includegraphics[width=0.145\textwidth]{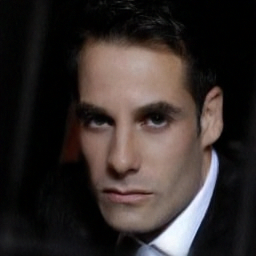} &
    \includegraphics[width=0.145\textwidth]{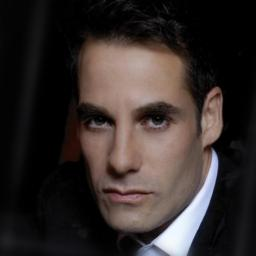} & &
    \includegraphics[width=0.145\textwidth]{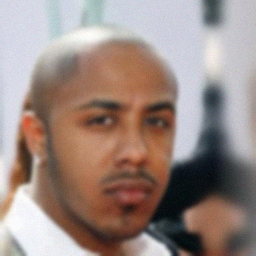} &
    \includegraphics[width=0.145\textwidth]{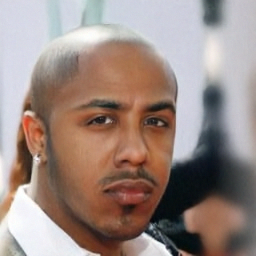} &
    \includegraphics[width=0.145\textwidth]{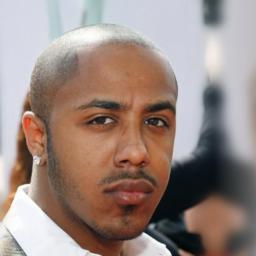} \\
  \end{tabular}
  \caption{CelebA-HQ reconstructions from \ours{} for $4\times$ super-resolution (top) and Gaussian deblurring (bottom).
Each triplet shows the observation, the \ours{} reconstruction, and the ground truth.}
  \label{fig:celebahq-recon}
\end{figure}

\subsection{Expanded CIFAR-10 inpainting comparison}\label{app:expanded-inpainting}

Table~\ref{tab:expanded-inpainting} adds DAPS, SITCOM, and TDS to the CIFAR-10 inpainting comparison.
The methods use the prior families and settings in Appendix~\ref{app:additional-baseline-settings}.
\ours{} attains the highest SSIM for both masks.

\begin{table}[!htbp]
  \centering
  \caption{CIFAR-10 inpainting quality and cost.
Time is per image.}
  \label{tab:expanded-inpainting}
  \setlength{\tabcolsep}{6pt}
  \begin{adjustbox}{max width=\textwidth}
  \begin{tabular}{@{}lcccc@{}}
    \toprule
    Method & Time (s) & Peak memory (GB) & SSIM, half $\uparrow$ & SSIM, random $\uparrow$ \\
    \midrule
    DAPS & $20.21$ & $0.57$ & $0.626$ & $0.908$ \\
    SITCOM & $3.82$ & $0.71$ & $0.645$ & $0.907$ \\
    TDS & $2.26$ & $1.78$ & $0.630$ & $0.927$ \\
    PnP-Flow & $5.08$ & $0.60$ & $0.623$ & $0.940$ \\
    FMPS & $2.27$ & $1.64$ & $0.240$ & $0.554$ \\
    TFTF & $2.56$ & $1.86$ & $0.532$ & $0.875$ \\
    DPS-SB & $2.44$ & $1.59$ & $0.518$ & $0.782$ \\
    \ours{} & $2.42$ & $1.59$ & $\mathbf{0.655}$ & $\mathbf{0.943}$ \\
    \bottomrule
  \end{tabular}
  \end{adjustbox}
\end{table}

\subsection{Calibration of the endpoint predictor}\label{app:calibration}

We measure the gap between the plug-in predictor $\hmu{1| t}$ and $\EE[X_1| X_t]$ along the trajectory.
For a state $x_t$, we estimate $\EE[X_1| X_t=x_t]$ by the Monte Carlo average $\bar\mu_{\mathrm{MC}}$ of $128$ suffix trajectories of the prior dynamics started at $(t,x_t)$ and report $\mathrm{err}^2(\hmu{1| t})=\|\hmu{1| t}(x_t)-\bar\mu_{\mathrm{MC}}\|^2/d$.
Table~\ref{tab:calibration} shows that the error is largest near $t=0$ and decays with $t$.
Writing $b^\star$ for the SB drift and $b_{\theta,g}$ for the sampler drift, the gap has the form
\[
\hmu{1\mid t}(x)-\EE[X_1\mid X_t=x]
=\EE\Big[\int_t^1\big(b^\star_t(x)-b_{\theta,g}(u,X_u)\big)\,\dd u\;\Big|\;X_t=x\Big],
\]
so as $t\to1$, the interval $(t,1)$ shrinks and $X_u$ stays near $x$ within it, and the integrand and the range of integration vanish together.

\begin{table}[!htbp]
  \caption{Calibration error $\mathrm{err}^2(\hmu{1| t})$ of the endpoint predictor against Monte Carlo estimates of $\EE[X_1| X_t]$ at different times $t$.}
  \label{tab:calibration}
  \centering
  
  \setlength{\tabcolsep}{6pt}
  \begin{adjustbox}{max width=\textwidth}
  \begin{tabular}{@{}lccccccc@{}}
    \toprule
    $t$ & 0.05 & 0.1 & 0.3 & 0.5 & 0.7 & 0.9 & 0.95 \\
    \midrule
    Toy example & 0.5572 & 0.4587 & 0.1303 & 0.0193 & 0.0034 & 0.00066 & 0.00031 \\
    CIFAR-10 & 0.00334 & 0.00164 & 0.00129 & 0.00157 & 0.00072 & 0.00044 & 0.00052 \\
    \bottomrule
  \end{tabular}
  \end{adjustbox}
\end{table}

\subsection{SMC diagnostics}\label{app:ess}

We record the ESS and resampling trajectories for CIFAR-10 class-conditional sampling and CelebA-HQ text-prompted image generation in Figures~\ref{fig:full-ess-cifar} and~\ref{fig:full-ess-celeba}, respectively.
The ESS threshold is $0.5$; \ours{} enables resampling after the first $10\%$ of the path, while TFTF resamples only in its stochastic phase.
IS has uniform intermediate weights and weights samples only at the endpoint, so its intermediate $\mathrm{ESS}/K$ is one by construction.

Table~\ref{tab:ess} summarizes for both datasets.
On CIFAR-10, \ours{} maintains a comparable time-averaged ESS to TDS with substantially fewer resampling events.
On CelebA-HQ, \ours{} has a higher terminal ESS than IS and TFTF, while TFTF has a higher time-averaged ESS and fewer resampling events than \ours{}.

\begin{figure}[!t]
  \centering
  \includegraphics[width=\linewidth]{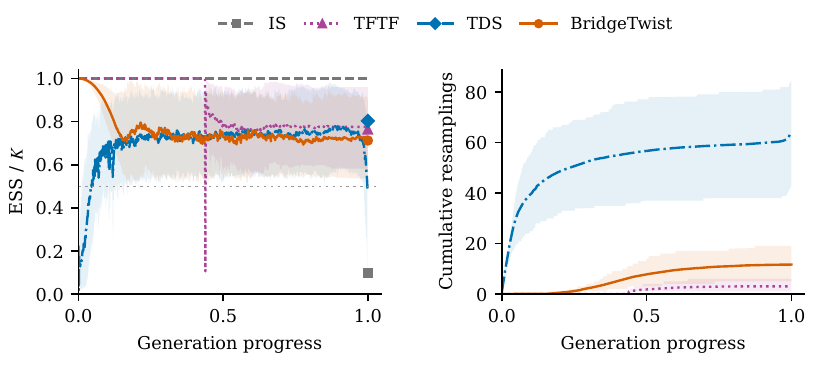}
  \caption{CIFAR-10 ESS and resampling trajectories at $K=64$ and $N=800$.
Curves are means over ten classes and ten repetitions per class; shading gives the 10th--90th percentile range across the 100 runs.
Left: pre-resampling $\mathrm{ESS}/K$, with the $0.5$ threshold and separate markers for terminal ESS after the final weight update.
Right: cumulative counts of resampling events.}
  \label{fig:full-ess-cifar}
\end{figure}

\begin{figure}[!t]
  \centering
  \includegraphics[width=\linewidth]{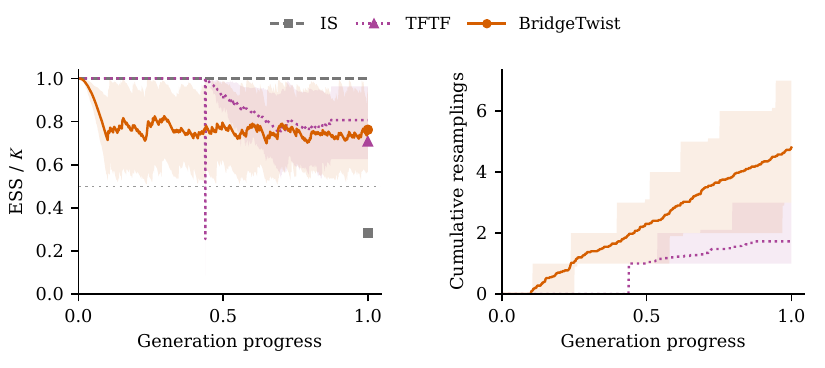}
  \caption{CelebA-HQ ESS and resampling trajectories at $K=25$ and $N=800$.
Curves are means over 40 attribute prompts with one run per prompt, and shading gives the 10th--90th percentile range across these prompt trials.
Left: pre-resampling $\mathrm{ESS}/K$, with the $0.5$ threshold and separate terminal ESS markers.
Right: cumulative counts of resampling events.}
  \label{fig:full-ess-celeba}
\end{figure}

\begin{table}[!htbp]
  \caption{SMC diagnostics along the full trajectories. A dash marks the absence of a compatible checkpoint.}
  \label{tab:ess}
  \centering
  \setlength{\tabcolsep}{3pt}
  \begin{adjustbox}{max width=\textwidth}
  \begin{tabular}{@{}lcccccc@{}}
    \toprule
    & \multicolumn{3}{c}{CIFAR-10} & \multicolumn{3}{c}{CelebA-HQ} \\
    \cmidrule(lr){2-4}\cmidrule(l){5-7}
    Method & step-avg.\ ESS/$K$ & resamplings & final ESS/$K$ & step-avg.\ ESS/$K$ & resamplings & final ESS/$K$ \\
    \midrule
    IS & $1.000$ & $0$ & $0.096$ & $1.000$ & $0$ & $0.282$ \\
    TFTF & $0.875$ & $3.05$ & $0.763$ & $0.902$ & $1.73$ & $0.708$ \\
    TDS & $0.705$ & $64.25$ & $0.803$ & -- & -- & -- \\
    \ours{} & $0.756$ & $11.67$ & $0.712$ & $0.770$ & $4.80$ & $0.762$ \\
    \bottomrule
  \end{tabular}
  \end{adjustbox}
\end{table}

\subsection{Tempering in TDS and TFTF}\label{app:tempering-transfer}

We apply our linear tempering schedule $\beta(t)=t$ to the twists in TDS and TFTF, retaining each method's prior and guided-proposal construction.
Table~\ref{tab:tempering-transfer} reports the comparisons on CIFAR-10.
All runs use $K=64$ particles and $N=800$ steps.
Tempering improves TDS, consistent with correcting its plug-in over-confidence, and it slightly reduces TFTF's accuracy, consistent with TFTF's early unconditional phase already suppressing guidance. 
TDS with tempering remains below \ours{}, which additionally uses an SB-specific bridge-shaped inference schedule, a flexibility with no counterpart under TDS.

\begin{table}[!htbp]
  \centering
  \caption{TDS, TFTF, their linearly tempered variants, and \ours{} on CIFAR-10.}
  \label{tab:tempering-transfer}
  \begin{adjustbox}{max width=\textwidth}
  \begin{tabular}{@{}lccc@{}}
    \toprule
    Method & ACC$_L$ $\uparrow$ & ACC$_E$ $\uparrow$ & KID $\downarrow$ \\
    \midrule
    TDS & $0.947$ & $0.869$ & $0.129$ \\
    TDS + tempering & $0.995$ & $0.951$ & $0.068$ \\
    TFTF & $0.986$ & $0.934$ & $0.057$ \\
    TFTF + tempering & $0.984$ & $0.923$ & $0.051$ \\
    \ours{} & $\mathbf{0.997}$ & $\mathbf{0.979}$ & $\mathbf{0.040}$ \\
    \bottomrule
  \end{tabular}
  \end{adjustbox}
\end{table}

\subsection{Second-order twist extension}\label{app:second-order-twist}

The first-order twist evaluates the likelihood at the conditional endpoint mean estimate and discards conditional endpoint uncertainty.
A second-order approximation represents the conditional endpoint law by a Gaussian with mean $\hmu{1| t}(x)$ and covariance $\widehat\Sigma_{1| t}(x)$.
For a linear Gaussian observation $y=\A x_1+\eta$, where $y\in\mathbb R^m$, $x_1\in\mathbb R^d$, $\A\in\mathbb R^{m\times d}$, and $\eta\sim\cN(0,R)$ with $R\succ0$, integrating the likelihood gives
\[
  \hh_t^{(2)}(x)
  =\int \cN(y;\A z,R)\,
    \cN(z;\hmu{1\mid t}(x),\widehat\Sigma_{1\mid t}(x))\,\mathrm dz
  =\cN(y;\A\hmu{1\mid t}(x),
    R+\A\widehat\Sigma_{1\mid t}(x)\A^\top).
\]
Thus, endpoint uncertainty increases the covariance in the predictive likelihood.

For the two-dimensional toy in Appendix~\ref{app:toy-benchmark}, $\A=\I$ and $R=\sigma_y^2\I$.
We compare the first-order twist with second-order twists using either the exact conditional covariance or a Jacobian-based covariance estimate.
For the estimated variant, let $\mustar{M| t}(x)=x+\TT_t^2\sstar_t(x)$ denote $\EE[M_t| X_t=x]$, where $\TT_t^2=\sigma^2t(1-t)$.
We use
\begin{align*}
  \Sigma_{M\mid t}(x)
  &:=\Cov(M_t\mid X_t=x)
    =\TT_t^2\big[\I+\TT_t^2\nabla_x\sstar_t(x)\big],\\
  G_t(x)
  &:=\I+(1-t)\bigg[\nabla_z\vstar_t(z)
    +\frac{\sigma^2}{2}(1-2t)\nabla_z\sstar_t(z)\bigg]_{z=\mustar{M\mid t}(x)},\\
  \widehat\Sigma_{1\mid t}(x)
  &:=G_t(x)\Sigma_{M\mid t}(x)G_t(x)^\top.
\end{align*}
All Jacobians are evaluated analytically in this two-dimensional experiment.

Table~\ref{tab:second-order-twist} reports the comparison at $K=2000$ and $N=100$, with 50 independent repetitions per observation-noise level.
Both second-order variants improve the results at the tighter observation-noise levels, while the differences are small at $\sigma_y\in\{0.5,1.0\}$.
Constructing the full covariance in high-dimensional image spaces requires additional Jacobian operations, so the image experiments retain the first-order twist.

\begin{table}[!htbp]
  \centering
  \caption{First- and second-order twists on the two-dimensional toy.
Entries are the square root of the mean MMD$^2$, with smaller values preferred.}
  \label{tab:second-order-twist}
  \begin{adjustbox}{max width=\textwidth}
  \begin{tabular}{@{}lccccc@{}}
    \toprule
    Twist / $\sigma_y$ & $1.0$ & $0.5$ & $0.3$ & $0.1$ & $0.05$ \\
    \midrule
    First-order & $0.032$ & $0.033$ & $0.134$ & $0.062$ & $0.059$ \\
    Second-order, exact covariance & $0.032$ & $0.031$ & $0.037$ & $0.025$ & $0.016$ \\
    Second-order, estimated covariance & $0.033$ & $0.031$ & $0.037$ & $0.028$ & $0.020$ \\
    \bottomrule
  \end{tabular}
  \end{adjustbox}
\end{table}

\subsection{Conditional transition with a non-Gaussian source}\label{app:scrna}

To illustrate the applicability of \ours{} to non-Gaussian sources, we evaluate \ours{} on the PHATE embryoid-body scRNA-seq data \citep{moon2019visualizing} (Mendeley v6n743h5ng).
We transport an early capture block (T4\_5C) to a late block (T6\_7D) while conditioning on a high proliferation-marker score at the endpoint, using $2048$ cells per block, $30$ principal components of $2000$ highly variable genes, $K=2048$, and $N=150$.
We use the posterior proxy inferred from the official Waddington-OT transport map \citep{schiebinger2019optimal} as the reference, which is independent of all samplers.

To target cells with high proliferation-marker expression, we construct a single score from the expression of MKI67, TOP2A, PCNA, CENPF, TYMS, and UBE2C.
We condition on a noisy observation of this score at its 95th percentile under the Waddington-OT descendant distribution.

Table~\ref{tab:scrna} reports the energy distance and the sliced Wasserstein-2 distance to the reference; \ours{} attains the lowest values of both, and TFTF degrades severely.
This stochastic transport from a non-Gaussian source illustrates an SB-specific use case for \ours{}.

\begin{table}[!htbp]
  \caption{Conditional transition on PHATE scRNA-seq data with a non-Gaussian source. Bold marks the best value in each column.}
  \label{tab:scrna}
  \centering
  
  \setlength{\tabcolsep}{8pt}
  \begin{adjustbox}{max width=\textwidth}
  \begin{tabular}{@{}lcc@{}}
    \toprule
    Method & Energy distance $\downarrow$ & Sliced $W_2$ $\downarrow$ \\
    \midrule
    FMPS & 0.0650 & 0.0599 \\
    TFTF & 3.1958 & 0.4919 \\
    DPS-SB & 0.0463 & 0.0550 \\
    IS & 0.0701 & 0.0610 \\
    \ours{} & \textbf{0.0374} & \textbf{0.0470} \\
    \bottomrule
  \end{tabular}
  \end{adjustbox}
\end{table}

\clearpage
\section*{NeurIPS Paper Checklist}

\begin{enumerate}

\item {\bf Claims}
    \item[] Question: Do the main claims made in the abstract and introduction accurately reflect the paper's contributions and scope?
    \item[] Answer: \answerYes{}.
    \item[] Justification: The claims are stated in the abstract and Section~\ref{sec:intro}; the theoretical scope is specified in Section~\ref{sec:method:theory}, and limitations are discussed in Section~\ref{sec:discussion}.

\item {\bf Limitations}
    \item[] Question: Does the paper discuss the limitations of the work performed by the authors?
    \item[] Answer: \answerYes{}.
    \item[] Justification: Section~\ref{sec:discussion} discusses the restriction to a Brownian reference, limitations of the plug-in twist, errors from prior learning and discretization, and the additional pretraining cost.

\item {\bf Theory assumptions and proofs}
    \item[] Question: For each theoretical result, does the paper provide the full set of assumptions and a complete (and correct) proof?
    \item[] Answer: \answerYes{}.
    \item[] Justification: Appendix~\ref{app:proofs} gives the assumptions and proofs for Propositions~\ref{thm:inversion} and~\ref{prop:fokker-planck}, Theorem~\ref{thm:exactness}, and Remark~\ref{rem:beta-star}.

    \item {\bf Experimental result reproducibility}
    \item[] Question: Does the paper fully disclose all the information needed to reproduce the main experimental results of the paper to the extent that it affects the main claims and/or conclusions of the paper (regardless of whether the code and data are provided or not)?
    \item[] Answer: \answerYes{}.
    \item[] Justification: Section~\ref{sec:experiments} specifies tasks, baselines, particle counts, sampler steps, and metrics; Appendix~\ref{app:implementation} additionally provides analytic fields for the toy benchmark, the unconditional pretraining setup, conditional likelihoods, hyperparameters and settings, and compute resources.

\item {\bf Open access to data and code}
    \item[] Question: Does the paper provide open access to the data and code, with sufficient instructions to faithfully reproduce the main experimental results, as described in supplemental material?
    \item[] Answer: \answerYes{}.
    \item[] Justification: A code release with pretrained checkpoints and scripts to reproduce the main experimental results is provided in the supplementary material.

\item {\bf Experimental setting/details}
    \item[] Question: Does the paper specify all the training and test details (e.g., data splits, hyperparameters, how they were chosen, type of optimizer) necessary to understand the results?
    \item[] Answer: \answerYes{}.
    \item[] Justification: Section~\ref{sec:experiments} specifies the tasks, baselines, particle counts, sampler steps, and evaluation metrics.
    Appendix~\ref{app:pretraining} describes the pretraining parameterization; Appendix~\ref{app:implementation} gives the pretrained models, likelihoods, sampler hyperparameters, baseline settings, and metric averaging conventions.

\item {\bf Experiment statistical significance}
    \item[] Question: Does the paper report error bars suitably and correctly defined or other appropriate information about the statistical significance of the experiments?
    \item[] Answer: \answerYes{}.
    \item[] Justification: Section~\ref{sec:experiments} reports the standard deviations of metrics based on multiple independent replications.

\item {\bf Experiments compute resources}
    \item[] Question: For each experiment, does the paper provide sufficient information on the computer resources (type of compute workers, memory, time of execution) needed to reproduce the experiments?
    \item[] Answer: \answerYes{}.
    \item[] Justification: Appendix~\ref{app:implementation} states that all experiments were conducted on one NVIDIA RTX 5090 GPU.
    
\item {\bf Code of ethics}
    \item[] Question: Does the research conducted in the paper conform, in every respect, with the NeurIPS Code of Ethics \url{https://neurips.cc/public/EthicsGuidelines}?
    \item[] Answer: \answerYes{}.
    \item[] Justification: The work uses existing public image datasets and the PHATE embryoid-body single-cell RNA-sequencing dataset for computational experiments, without recruiting participants or collecting new human-subject data.

\item {\bf Broader impacts}
    \item[] Question: Does the paper discuss both potential positive societal impacts and negative societal impacts of the work performed?
    \item[] Answer: \answerNA{}.
    \item[] Justification: The paper presents a training-free sampler for pretrained Schr\"odinger bridge generative models. It inherits the broader-impact profile of generic conditional generative modeling without introducing new specific risks.
    
\item {\bf Safeguards}
    \item[] Question: Does the paper describe safeguards that have been put in place for responsible release of data or models that have a high risk for misuse (e.g., pre-trained language models, image generators, or scraped datasets)?
    \item[] Answer: \answerNA{}.
    \item[] Justification: The paper does not release a high-risk foundation model or a scraped dataset.

\item {\bf Licenses for existing assets}
    \item[] Question: Are the creators or original owners of assets (e.g., code, data, models), used in the paper, properly credited and are the license and terms of use explicitly mentioned and properly respected?
    \item[] Answer: \answerYes{}.
    \item[] Justification: All existing methods, datasets, and pretrained checkpoints used in this work are properly cited in Sections~\ref{sec:intro}--\ref{sec:experiments} and Appendix~\ref{app:implementation}; their licenses are documented in the README of the supplementary code release.

\item {\bf New assets}
    \item[] Question: Are new assets introduced in the paper well documented and is the documentation provided alongside the assets?
    \item[] Answer: \answerYes{}.
    \item[] Justification: The supplementary release contains the reproduction code and the pretrained SB checkpoints used to produce the reported results, documented in the README. No new datasets or benchmark assets are introduced.

\item {\bf Crowdsourcing and research with human subjects}
    \item[] Question: For crowdsourcing experiments and research with human subjects, does the paper include the full text of instructions given to participants and screenshots, if applicable, as well as details about compensation (if any)? 
    \item[] Answer: \answerNA{}.
    \item[] Justification: The experiments do not involve crowdsourcing or human subject studies.

\item {\bf Institutional review board (IRB) approvals or equivalent for research with human subjects}
    \item[] Question: Does the paper describe potential risks incurred by study participants, whether such risks were disclosed to the subjects, and whether Institutional Review Board (IRB) approvals (or an equivalent approval/review based on the requirements of your country or institution) were obtained?
    \item[] Answer: \answerNA{}.
    \item[] Justification: The experiments do not involve crowdsourcing or human subject studies.

\item {\bf Declaration of LLM usage}
    \item[] Question: Does the paper describe the usage of LLMs if it is an important, original, or non-standard component of the core methods in this research? Note that if the LLM is used only for writing, editing, or formatting purposes and does \emph{not} impact the core methodology, scientific rigor, or originality of the research, declaration is not required.
    \item[] Answer: \answerNA{}.
    \item[] Justification: The core methodology does not use LLMs as an original or non-standard component.

\end{enumerate}

\end{document}